\PassOptionsToPackage{table}{xcolor}
\documentclass{article} 

\usepackage{microtype}
\usepackage{graphicx}
\usepackage{subfigure}
\usepackage{booktabs} 

\usepackage{hyperref}

\usepackage[accepted]{icml2024}

\usepackage{hyperref}
\usepackage{url}

\usepackage{amsmath}
\usepackage{amssymb}
\usepackage{mathtools}
\usepackage{amsthm}

\usepackage{multirow}
\usepackage{multicol}
\usepackage{booktabs}

\usepackage{tikz}

\newcommand{\squishlist}{
 \begin{list}{$\bullet$}
  { \setlength{\itemsep}{0pt}
     \setlength{\parsep}{3pt}
     \setlength{\topsep}{3pt}
     \setlength{\partopsep}{0pt}
     \setlength{\leftmargin}{1.5em}
     \setlength{\labelwidth}{1em}
     \setlength{\labelsep}{0.5em} } }

\newcommand{\squishlisttwo}{
 \begin{list}{$\bullet$}
  { \setlength{\itemsep}{0pt}
    \setlength{\parsep}{0pt}
    \setlength{	opsep}{0pt}
    \setlength{\partopsep}{0pt}
    \setlength{\leftmargin}{2em}
    \setlength{\labelwidth}{1.5em}
    \setlength{\labelsep}{0.5em} } }

\newcommand{\squishend}{
  \end{list}  }

\usepackage{algorithm}
\usepackage[noend]{algpseudocode}

\renewcommand{\algorithmicrequire}{\textbf{Input:}}
\renewcommand{\algorithmicensure}{\textbf{Output:}}

\newcommand{\R}[0]{\mathbb{R}}

\definecolor{dark2green}{rgb}{0.1, 0.65, 0.3}
\definecolor{dark2orange}{rgb}{0.9, 0.4, 0.}
\definecolor{dark2purple}{rgb}{0.4, 0.4, 0.8}
\newcommand{\first}[1]{\textbf{\textcolor{dark2green}{#1}}}
\newcommand{\second}[1]{\textbf{\textcolor{dark2orange}{#1}}}
\newcommand{\third}[1]{\textbf{\textcolor{dark2purple}{#1}}}

\renewcommand{\algorithmicrequire}{\textbf{Input:}}
\renewcommand{\algorithmicensure}{\textbf{Output:}}

\definecolor{c1}{HTML}{6E001B} 
\definecolor{c2}{HTML}{6E001B} 
\definecolor{c3}{HTML}{5D9761} 
\definecolor{myblue}{HTML}{7BB2DD} 
\definecolor{mygray}{HTML}{DBE2E9}

\newcommand{\head}[1]{\vspace{1.7mm}\noindent{\textcolor{c1}{\bf #1.}}}

\usepackage[capitalize,noabbrev]{cleveref}

\theoremstyle{plain}
\newtheorem{theorem}{Theorem}[section]

\theoremstyle{definition}

\theoremstyle{remark}

\usepackage[textsize=tiny]{todonotes}

\icmltitlerunning{Graph Mamba}

\begin{document}

\icmltitle{Graph Mamba: Towards Learning on Graphs with State Space Models}



\icmlsetsymbol{equal}{*}

\begin{icmlauthorlist}
\icmlauthor{Ali Behrouz}{equal,cornell}
\icmlauthor{Farnoosh Hashemi}{equal,cornell}
\end{icmlauthorlist}

\icmlaffiliation{cornell}{Cornell University, Ithaca, USA}

\icmlcorrespondingauthor{Ali Behrouz}{ab2947@cornell.edu}
\icmlcorrespondingauthor{Farnoosh Hashemi}{sh2574@cornell.edu}


\vskip 0.3in



 \printAffiliationsAndNotice{\icmlEqualContribution} 
\begin{abstract}
Graph Neural Networks (GNNs) have shown promising potential in graph representation learning. The majority of GNNs define a local message-passing mechanism, propagating information over the graph by stacking multiple layers. These methods, however, are known to suffer from two major limitations: over-squashing and poor capturing of long-range dependencies. Recently, Graph Transformers (GTs) emerged as a powerful alternative to Message-Passing Neural Networks (MPNNs). GTs, however, have quadratic computational cost, lack inductive biases on graph structures, and rely on complex Positional/Structural Encodings (SE/PE). In this paper, we show that while Transformers, complex message-passing, and SE/PE are sufficient for good performance in practice, neither is necessary. Motivated by the recent success of State Space Models (SSMs), such as Mamba, we present Graph Mamba Networks (GMNs), a general framework for a new class of GNNs based on selective SSMs. We discuss and categorize the new challenges when adapting SSMs to graph-structured data, and present four required and one optional steps to design GMNs, where we choose (1) Neighborhood Tokenization, (2) Token Ordering, (3) Architecture of Bidirectional Selective SSM Encoder, (4) Local Encoding, and dispensable (5) PE and SE. We further provide theoretical justification for the power of GMNs. Experiments demonstrate that despite much less computational cost, GMNs attain an outstanding performance in long-range, small-scale, large-scale, and heterophilic benchmark datasets. The code is in \href{https://github.com/GraphMamba/GMN}{\textcolor{c1}{this link}}. 
\end{abstract}

\section{Introduction}\label{sec:introduction}
Recently, graph learning has become an important and popular area of study due to its impressive results in a wide range of applications, like neuroscience~\cite{brain-mixer}, social networks~\citep{fan2019graph}, molecular graphs~\citep{wang2021molecular}, etc. In recent years, Message-Passing Neural Networks (MPNNs), which iteratively aggregate neighborhood information to learn the node/edge representations, have been the dominant paradigm in machine learning on graphs~\citep{kipf2016semi, GAT, wu2020comprehensive, gutteridge2023drew}. They, however, have some inherent limitations, including over-squashing~\citep{di2023oversquash}, over-smoothing~\citep{rusch2023oversmoothing}, and poor capturing of long-range dependencies~\citep{long-range-data}. With the rise of Transformer architectures~\citep{transformer} and their success in diverse applications such as natural language processing~\citep{wolf2020transformers} and computer vision~\citep{liu2021swin}, their graph adaptations, so-called Graph Transformers (GTs), have gained popularity as the alternatives of MPNNs~\citep{yun2019graph, kim2022pure, GPS}.

Graph transformers have shown promising performance in various graph tasks, and their variants have achieved top scores in several graph learning benchmarks~\citep{hu2020open, long-range-data}. The superiority of GTs over MPNNs is often explained by MPNNs' bias towards encoding local structures~\citep{muller2023attending}, while a key underlying principle of GTs is to let nodes attend to all other nodes through a global attention mechanism~\citep{kim2022pure, yun2019graph}, allowing direct modeling of long-range interactions. Global attention, however, has weak inductive bias and typically requires incorporating information about nodes' positions to capture the graph structure~\citep{GPS, kim2022pure}. To this end, various positional and structural encoding schemes based on spectral and graph features have been introduced~\citep{kreuzer2021rethinking, kim2022pure, lim2023sign}. 

Despite the fact that GTs with proper positional encodings (PE) are universal approximators and provably more powerful than any Weisfeiler-Lehman isomorphism test (WL test)~\citep{kreuzer2021rethinking}, their applicability to large-scale graphs is hindered by their poor scalability. That is, the standard global attention mechanism on a graph with $n$ nodes incurs both time and memory complexity of $\mathcal{O}(n^2)$, quadratic in the input size, making them infeasible on large graphs. To overcome the high computational cost, inspired by linear attentions~\citep{zaheer2020big}, sparse attention mechanisms on graphs attracts attention~\citep{GPS, shirzad2023exphormer}. For example, Exphormer~\citep{shirzad2023exphormer} suggests using expander graphs, global connectors, and local neighborhoods as three patterns that can be incorporated in GTs, resulting in a sparse and efficient attention. Although sparse attentions partially overcome the memory cost of global attentions, GTs based on these sparse attentions~\citep{GPS, shirzad2023exphormer} still might suffer from quadratic time complexity. That is, they require costly PE (e.g., Laplacian eigen-decomposition) and structural encoding (SE) to achieve their best performance, which can take $\mathcal{O}(n^2)$ to compute.

Another approach to improve GTs’ high computational cost is to use subgraph tokenization~\citep{chen2023nagphormer, zhao2021gophormer, kuang2021coarformer, baek2021accurate, graph-mlpmixer}, where tokens (a.k.a patches) are small subgraphs extracted with a pre-defined strategy. Typically, these methods obtain the initial representations of the subgraph tokens by passing them through an MPNN. Given $k$ extracted subgraphs (tokens), the time complexity of these methods is $\mathcal{O}(k^2)$, which is more efficient than typical GTs with node tokenization. Also, these methods often do not rely on complex PE/SE, as their tokens (subgraphs) inherently carry inductive bias. These methods, however, have two major drawbacks: (1) To achieve high expressive power, given a node, they usually require at least a subgraph per each remaining node~\citep{subgraph-expressiveness, bar-shalom2023subgraphormer}, meaning that $k \in \mathcal{O}(n)$ and so the time complexity is $\mathcal{O}(n^2)$.  (2) Encoding subgraphs via MPNNs can transmit all their challenges of over-smoothing and over-squashing, limiting their applicability to heterophilic and long-range graphs. 

Recently, Space State Models (SSMs), as an alternative of attention-based sequence modeling architectures like Transformers have gained increasing popularity due to their efficiency~\citep{zhang2023effectively, nguyen2023hyenadna}. They, however, do not achieve competitive performance with Transformers due to their limits in input-dependent context compression in sequence models, caused by their time-invariant transition mechanism. To this end, \citet{gu2023mamba} present Mamba, a selective state space model that uses recurrent scans along with a selection mechanism to control which part of the sequence can flow into the hidden states. This selection can simply be interpreted as using data-dependent state transition mechanism (See \S\ref{sec:ssm} for a detailed discussion). Mamba outstanding performance in language modeling, outperforming Transformers of the same size and matching Transformers twice its size, motivates several recent studies to adapt its architecture for different data modalities~\citep{liu2024vmamba, yang2024vivim, zhu2024vision, ahamed2024mambatab}.

Mamba architecture is specifically designed for sequence data and the complex non-causal nature of graphs makes directly applying Mamba on graphs challenging. Further, natural attempts to replace Transformers with Mamba in existing GTs frameworks (e.g., GPS~\citep{GPS}, TokenGT \citep{kim2022pure}) results in suboptimal performance in both effectiveness and time efficiency (See \S\ref{sec:experiments} for evaluation and \S\ref{sec:motivations} for a detailed discussion). The reason is, contrary to Transformers that allows each node to interact with all the other nodes, Mamaba, due to its recurrent nature, only incorporates information about previous tokens (nodes) in the sequence. This introduces new challenges compared to GTs: (1) The new paradigm requires token ordering that allows the model take advantage of the provided positional information as much as possible. (2) The architecture design need to be more robust to permutation than a pure sequential encoder (e.g., Mamba). (3) While the quadratic time complexity of attentions can dominate the cost of PE/SE in GTs, complex PE/SE (with $\mathcal{O}(n^2)$ cost) can be a bottleneck for scaling Graph Mamba on large graphs.

\head{Contributions}  To address all the abovementioned limitations, we present Graph Mamba Networks (GMNs), a new class of machine learning on graphs based on state space models (Figure~\ref{fig:framework} shows the schematic of the GMNs). In summary our contributions are:
\squishlist   
    \item \head{Recipe for Graph Mamba Networks} We discuss new challenges of GMNs compared to GTs in architecture design and motivate our recipe with four required and one optional steps to design GMNs. In particular, its steps are (1) Tokenization, (2) Token Ordering, (3) Local Encoding, (4) Bidirectional Selective SSM Encoder and dispensable (5) PE and SE.
    \item \head{An Efficient Tokenization for Bridging Frameworks} Literature lacks a common foundation about what constitutes a good tokenization. Accordingly, architectures are required to choose either node- or subgraph-level tokenization, while each of which has its own (dis)advantages, depending on the data. We present a graph tokenization process that not only is fast and efficient, but it also bridges the node- and subgraph-level tokenization methods using a single parameter. Moreover, the presented tokenization has implicit order, which is specially important for sequential encoders like SSMs.
    \item \head{New Bidirectional SSMs for Graphs} Inspired by Mamba, we design a SSM architecture that scans the input sequence in two different directions, making the model more robust to permutation, which is particularly important when we do not use implicitly ordered tokenization on graphs. 
    \item \head{Theoretical Justification} We provide theoretical justification for the power of GMNs and show that they are universal approximator of any functions on graphs. We further show that GMNs using proper PE/SE is more expressive than any WL test, matching GTs in this manner. 
    \item \head{Outstanding Performance and New Insights} Our experimental evaluations demonstrate that GMNs attain an outstanding performance in long-range, small-scale, large-scale, and heterophilic benchmark datasets, while consuming less GPU memory. These results show that while Transformers, complex message-passing, and SE/PE are sufficient for good performance in practice, neither is necessary. We further perform ablation study and validate the contribution of each architectural choice. 
\squishend

\begin{figure}
    \centering
    \hspace*{-5ex}\includegraphics[width=1.1\linewidth]{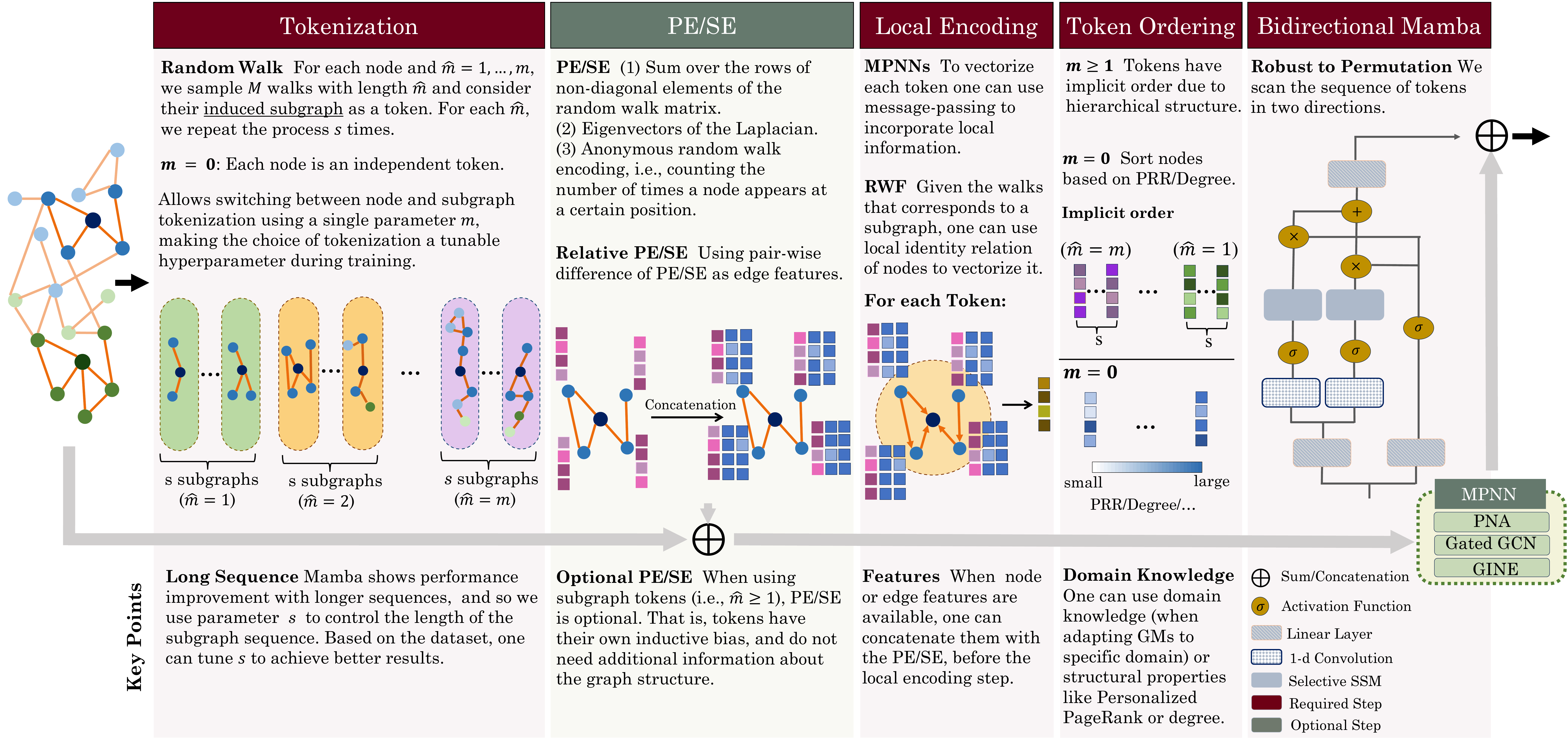}
    \vspace{-3ex}
    \caption{Schematic of the GMNs with four required and one optional steps: (1) Tokenization: the graph is mapped into a sequence of tokens ($m \geq 1$: subgraph and $m = 0$: node tokenization) (2) (\underline{Optional Step}) PE/SE: inductive bias is added to the architecture using information about the position of nodes and the strucutre of the graph. (3) Local Encoding: local structures around each node are encoded using a subgraph vectorization mechanism. (4) Token Ordering: the sequence of tokens are ordered based on the context. (\underline{Subgraph tokenization ($m \geq 1$) has implicit order} and does not need this step). (5) (Stack of) Bidirectional Mamba: it scans and selects relevant nodes or subgraphs to flow into the hidden states. $^\dagger$ In this figure, the last layer of bidirectional Mamba, which performs as a readout on all nodes, is omitted for simplicity.}
    \label{fig:framework}
\end{figure}

\section{Related Work and Backgrounds}\label{sec:RW}
To situate GMNs in a broader context, we discuss four relevant types of machine learning methods:

\subsection{Message-Passing Neural Networks}
Message-passing neural networks are a class of GNNs that iteratively aggregate local neighborhood information to learn the node/edge representations~\citep{kipf2016semi}. MPNNs have been the dominant paradigm in machine learning on graphs, and attracts much attention, leading to various powerful architectures, e.g., GAT~\citep{GAT}, GCN~\citep{henaff2015deep, kipf2016semi}, GatedGCN~\citep{bresson2017residual}, GIN~\citep{xu2018how}, etc. Simple MPNNs, however, are known to suffer from some major limitations including: (1) limiting their expressivity to the 1-WL isomorphism test~\citep{xu2018how}, (2) over-smoothing~\citep{rusch2023oversmoothing}, and (3) over-squashing~\citep{alon2021on, di2023oversquash}. Various methods have been developed to augment MPNNs and overcome such issues, including higher-order GNNs~\citep{morris2019weisfeiler, morris2020weisfeiler}, graph rewiring~\citep{gutteridge2023drew, rodriguez2022diffwire}, adaptive and cooperative GNNs~\citep{errica2023adaptive, finkelshtein2023cooperative}, and using additional features~\citep{sato2021random, murphy2019relational}.

\subsection{Graph Transformers}
With the rise of Transformer architectures~\citep{transformer} and their success in diverse applications such as natural language processing~\citep{wolf2020transformers} and computer vision~\citep{liu2021swin}, their graph adaptations have gained popularity as the alternatives of MPNNs~\citep{yun2019graph, kim2022pure, GPS}. Using a full global attention, GTs consider each pair of nodes connected~\citep{yun2019graph} and so are expected to overcome the problems of over-squashing and over-smoothing in MPNNs~\citep{kreuzer2021rethinking}. GTs, however, have weak inductive bias and needs proper positional/structural encoding to learn the structure of the graph~\cite{kreuzer2021rethinking, GPS}. To this end, various studies have focused on designing powerful positional and structural encodings~\citep{wang2022equivariant, ying2021transformers, kreuzer2021rethinking, shiv2019novel}.

\head{Sparse Attention}
While GTs have shown outstanding performance in different graph tasks on small-scale datasets (up to 10K nodes), their quadratic computational cost, caused by their full global attention, has limited their applicability to large-scale graphs \citep{GPS}. Motivated by linear attention mechanisms (e.g., BigBird~\citep{zaheer2020big} and Performer~\citep{choromanski2021rethinking}), which are designed to overcome the same scalability issue of Transformers on long sequences, using sparse Transformers in GT architectures has gained popularity~\citep{GPS, shirzad2023exphormer, kong2023goat, liu2023gapformer, wu2023kdlgt}. The main idea of sparse GTs models is to restrict the attention pattern, i.e., the pairs of nodes that can interact with each other. As an example, \citet{shirzad2023exphormer} present Exphormer, the graph adaption of BigBird that uses three sparse patterns of (1) expander graph attention,  (2) local attention among neighbors, and (3) global attention by connecting virtual nodes to all non-virtual nodes.

\head{Subgraph Tokenization} 
Another method to overcome GTs’ high computational cost is to use subgraph tokenization~\citep{chen2023nagphormer, zhao2021gophormer, baek2021accurate, graph-mlpmixer}, where tokens are small subgraphs extracted with a pre-defined strategy. These subgraph tokenization strategies usually are $k$-hop neighborhood (given a fixed $k$)~\citep{nguyen2022universal, hussain2022global, park2022deformable}, learnable sample of neighborhood~\citep{zhang2022hierarchical}, ego-networks~\citep{zhao2021gophormer}, hierarchical $k$-hop neighborhoods~\citep{chen2023nagphormer}, graph motifs~\citep{rong2020self}, and graph partitions~\citep{graph-mlpmixer}. To vectorize each token, subgraph-based GT methods typically rely on MPNNs, making them vulnerable to over-smoothing and over-squashing. Most of them also use a fixed neighborhood of each node, missing the hierarchical structure of the graph. The only exception is NAGphormer~\citep{chen2023nagphormer} that uses all $k= 1, \dots, K$-hop neighborhoods of each node as its corresponding tokens. Although this tokenization lets the model learn the hierarchical structure of the graph, by increasing the hop of the neighborhood, its tokens become exponentially larger, limiting its ability to scale to large graphs.

\subsection{State Space Models}\label{sec:ssm}
State Space Models (SSMs), a type of sequence models, are usually known as linear time-invariant systems that map input sequence $x(t) \in \R^L$ to response sequence $y(t) \in \R^L$~\citep{aoki2013state}. Specifically, SSMs use a latent state $h(t) \in \R^{N \times L}$, evolution parameter $\mathbf{A}\in \R^{N\times N}$, and projection parameters $\mathbf{B} \in \R^{N\times 1}, \mathbf{C}\in \R^{1 \times N}$ such that:
\begin{align}\nonumber
    &h'(t) = \mathbf{A} \: h(t) + \mathbf{B} \: x(t), \\
    & y(t) = \mathbf{C} \: h(t).
\end{align}
Due to the hardness of solving the above differential equation in deep learning settings, discrete space state models~\citep{gu2020hippo, zhang2023effectively} discretize the above system using a parameter $\boldsymbol{\Delta}$:
\begin{align}\nonumber
    &h_t = \bar{\mathbf{A}} \: h_{t-1} + \bar{\mathbf{B}} \: x_t,\\ \label{eq:ssm1}
    &y_t = \mathbf{C} \: h_t, 
\end{align}
where 
\vspace{-2ex}
\begin{align}\nonumber
    &\bar{\mathbf{A}} = \exp\left( \boldsymbol{\Delta} \mathbf{A} \right),\\ \label{eq:ssm_disc}
    &\bar{\mathbf{B}} = \left( \boldsymbol{\Delta} \mathbf{A} \right)^{-1} \left( \exp \left( \boldsymbol{\Delta} \mathbf{A} - I \right) \right) \: . \: \boldsymbol{\Delta} \mathbf{B}.
\end{align}
\citet{gu2020hippo} shows that discrete-time SSMs  are equivalent to the following convolution:
\begin{align}\nonumber
    &\bar{\mathbf{K}} = \left( \bar{\mathbf{C}} \bar{\mathbf{B}}, \bar{\mathbf{C}} \bar{\mathbf{A}} \bar{\mathbf{B}}, \dots, \bar{\mathbf{C}} \bar{\mathbf{A}}^{L-1} \bar{\mathbf{B}} \right),\\ \label{eq:ssm2}
    &y = x \ast \bar{\mathbf{K}},
\end{align}
and accordingly can be computed very efficiently. Structured state space models (S4), another type of SSMs, are efficient alternatives of attentions and have improved efficiency and scalability of SSMs using reparameterization~\citep{gu2022efficiently, fu2023hungry, nguyen2023hyenadna}. SSMs show promising performance on timeseries data~\citep{zhang2023effectively, GraphS4mer}, Genomic sequence~\citep{nguyen2023hyenadna}, healthcare domain~\citep{gu2021combining}, and computer vision~\citep{gu2021combining, nguyen2022s4nd}. They, however, lack selection mechanism, causing missing the context as discussed by \citet{gu2023mamba}. Recently, \citet{gu2023mamba} introduce an efficient and powerful selective structured state space architecture, called \textsc{Mamba}, that uses recurrent scans along with a selection mechanism to control which part of the sequence can flow into the hidden states. The selection mechanism of Mamba can be interpreted as using data-dependent state transition mechanisms, i.e., making $\mathbf{B}, \mathbf{C},$ and $\boldsymbol{\Delta}$ as function of input $x_t$. Mamba outstanding performance in language modeling, outperforming Transformers of the same size and matching Transformers twice its size, motivates several recent studies to adapt its architecture for different data modalities and tasks~\citep{liu2024vmamba, yang2024vivim, zhu2024vision, ahamed2024mambatab, xing2024segmamba, liu2024swin, ma2024u}.

\section{Challenges \& Motivations: Transformers vs Mamba}\label{sec:motivations}
Mamba architecture is specifically designed for sequence data and the complex non-causal nature of graphs makes directly applying Mamba on graphs challenging. Based on the common applicability of Mamba and Transformers on tokenized sequential data, a straightforward approach to adapt Mamba for graphs is to replace Transformers with Mamba in GTs frameworks, e.g., TokenGT~\citep{kim2022pure} or GPS~\citep{GPS}. However, this approach might not fully take advantage of selective SSMs due to ignoring some of their special traits. In this section, we discuss new challenges for GMNs compared to GTs.

\head{Sequences vs 2-D Data}
It is known that the self-attentive architecture corresponds to a family of permutation equivariant functions~\citep{lee2019set, liu2020learning}. That is, the attention mechanism in Transformers~\citep{transformer} assumes a connection between each pair of tokens, regardless of their positions in the sequence, making it permutation equivariant. Accordingly, Transformers lack inductive bias and so properly positional encoding is crucial for their performance, whenever the order of tokens matter~\citep{transformer, liu2020learning}. On the other hand, Mamba is a sequential encoder and scans tokens in a recurrent manner (potentially less sensitive to positional encoding). Thus, it expects causal data as an input, making it challenging to be adapted to 2-D (e.g., images)~\citep{liu2024vmamba} or complex graph-structured data. Accordingly, while in graph adaption of Transformers mapping the graph into a sequence of tokens along with a positional/structural encodings were enough, sequential encoders, like SSMs, and more specifically Mamba, require an ordering mechanism for tokens.

Although this sensitivity to the order of tokens makes the adaption of SSMs to graphs challenging, it can be more powerful whenever the order matters. For example, learning the hierarchical structures in the neighborhood of each node ($k$-hops for $k=1, \dots, K$), which is implicitly ordered, is crucial in different domains~\citep{zhong2022long, lim2023brain}. Moreover, it provides the opportunity to use domain knowledge when the order matters~\citep{yu2020order}. In our proposed framework, we provide the opportunity for both cases: (1) using domain knowledge or structural properties (e.g., Personalized PageRank~\citep{page1998pagerank}) when the order matters, or (2) using implicitly ordered subgraphs (no ordering is needed). Furthermore, our bidirectional encoder scans nodes in two different directions, being capable of learning equivariance functions on the input, whenever it is needed.

\head{Long-range Sequence Modeling}
In graph domain, the sequence of tokens, either node, edge, or subgraph, can be counted as the context. Unfortunately, Transformer architecture, and more specifically GTs, are not scalable to long sequence. Furthermore, intuitively, more context (i.e., longer sequence) should lead to better performance; however, recently it has been empirically observed that many sequence models do not improve with longer context in language modeling~\citep{shi2023large}. Mamba, because of its selection mechanism, can simply filter irrelevant information and also reset its state at any time. Accordingly, its performance improves monotonically with sequence length~\citep{gu2023mamba}. To this end, and to fully take advantage of Mamba, one can map a graph or node to long sequences, possibly bags of various subgraphs. Not only the long sequence of tokens can provide more context, but it also potentially can improve the expressive power~\citep{bevilacqua2022equivariant}.

\head{Scalability}
Due to the complex nature of graph-structured data, sequential encoders, including Transformers and Mamba, require proper positional and structural encodings~\citep{GPS, kim2022pure}. These PEs/SEs, however, often have quadratic computational cost, which can be computed once before training. Accordingly, due to the quadratic time complexity of Transformers, computing these PEs/SEs was dominated and they have not been the bottleneck for training GTs. GMNs, on the other hand, have linear computational cost (with respect to both time and memory), and so constructing complex PEs/SEs can be their bottleneck when training on very large graphs. This bring a new challenge for GMNs, as they need to either (1) do not use PEs/SEs, or (2) use their more efficient variants to fully take advantage of SSMs efficiency. Our architecture design make the use of PE/SE optional and our empirical evaluation shows that GMNs without PE/SE can achieve competitive performance compared to methods with complex PEs/SEs.

\noindent
\textcolor{c1}{\textbf{Node or Subgraph}?}
In addition to the above new challenges, there is a lack of common foundation about what constitutes a good tokenization, and what differentiates them, even in GT frameworks. Existing methods use either node/edge~\citep{shirzad2023exphormer, GPS, kim2022pure}, or subgraph tokenization methods~\citep{chen2023nagphormer, zhao2021gophormer, graph-mlpmixer}. While methods with node tokenization are more capable of capturing long-range dependencies, methods with subgraph tokens have more ability to learn local neighborhoods, are less rely on PE/SE~\citep{chen2023nagphormer}, and are more efficient in practice. Our architecture design lets switching between node and subgraph tokenization using a single parameter $m$, making the choice of tokenization a tunable hyperparameter during training.

\section{Graph Mamba Networks}
In this section, we provide our five-step recipe for powerful, flexible, and scalable Graph Mamba Networks. Following the discussion about the importance of each step, we present our architecture. The overview of the GMN framework is illustrated in Figure~\ref{fig:framework}.

Throughout this section, we let $G = (V, E)$ be a graph, where $V= \{v_1, \dots, v_n\}$ is the set of nodes and $E \subseteq V\times V$ is the set of edges. We assume each node $v \in V$ has a feature vector $\mathbf{x}^{(0)}_v \in \mathbf{X}$, where $\mathbf{X} \in \R^{n \times d}$ is the feature matrix describing the attribute information of nodes and $d$ is the dimension of feature vectors. Given $v \in V$, we let $\mathcal{N}(v) = \{u | (v, u) \in E \}$ be the set of $v$'s neighbors. Given a subset of nodes $S \subseteq V$, we use $G[S]$ to denote the induced subgraph constructed by nodes in $S$, and $\mathbf{X}_{S}$ to denote the feature matrix describing the attribute information of nodes in $S$.

\subsection{Tokenization and Encoding}\label{sec:tokenization}
Tokenization, which is the process of mapping the graph into a sequence of tokens, is an inseparable part of adapting sequential encoders to graphs. As discussed earlier, existing methods use either node/edge~\citep{shirzad2023exphormer, GPS, kim2022pure}, or subgraph tokenization methods~\citep{chen2023nagphormer, zhao2021gophormer, graph-mlpmixer}, each of which has its own (dis)advantages. In this part, we present a new simple but flexible and effective neighborhood sampling for each node and discuss its advantages over existing subgraph tokenization. The main and high-level idea of our tokenization is to first, sample some subgraphs for each node that can represent the node's neighborhood structure as well as its local, and global positions in the graph. Then we vectorize (encode) these subgraphs to obtain the node representations.

\head{Neighborhood Sampling}
 Given a node $v \in V$, and two integers $m, M \geq 0$, for each $0 \leq \hat{m} \leq m$, we sample $M$ random walks started from $v$ with length $\hat{m}$. Let $T_{\hat{m}, i} (v)$ for $i = 0, \dots, M$ be the set of visited nodes in the $i$-th walk. We define the token corresponds to all walks with length $\hat{m}$ as:
 \begin{align}
     G\left[T_{\hat{m}}(v)\right] = G\left[\bigcup_{i = 0}^{M} T_{\hat{m}, i}(v)\right],
 \end{align}
 which is the union of all walks with length $\hat{m}$. One can interpret  $G[T_{\hat{m}}(v)]$ as the induced subgraph of a sample of $\hat{m}$-hop neighborhood of node $v$. At the end, for each node $v \in V$ we have the sequence of $G[T_0(v)], \dots, G[T_{m}(v)]$ as its corresponding tokens. 

Using random walks (with fixed length) or $k$-hop neighborhood of a node as its representative tokens has been discussed in several recent studies~\citep{ding2023recurrent, zhang2022hierarchical, chen2023nagphormer, zhao2021gophormer}. These methods, however, suffer from a subset of these limitations: (1) they use a fixed-length random walk~\citep{kuang2021coarformer}, which misses the hierarchical structure of the node's neighborhood. This is particularly important when the long-range dependencies of nodes are important. (2) they use all nodes in all $k$-hop neighborhoods~\citep{chen2023nagphormer, ding2023recurrent}, resulting in a trade-off between long-range dependencies and over-smoothing or over-squashing problems. Furthermore, the $k$-hop neighborhood of a well-connected node might be the whole graph, resulting in considering the graph as a token of a node, which is inefficient. Our neighborhood sampling approach addresses all these limitations. It sampled the fixed number of random walks with different lengths for all nodes, capturing hierarchical structure of the neighborhood while avoiding both inefficiency, caused by considering the entire graph, and over-smoothing and over squashing, caused by large neighborhood aggregation.

\noindent
\textcolor{c1}{\textbf{Why Not More Subgraphs}?}
As discussed earlier, empirical evaluation has shown that the performance of \emph{selective} state space models improves monotonically with sequence length~\citep{gu2023mamba}. Furthermore, their linear computational cost allow us to use more tokens, providing them more context. Accordingly, to fully take advantage of selective state space models, given an integer $s > 0$, we repeat the above neighborhood sampling process for $s$ times. Accordingly, for each node $v \in V$ we have a sequence of
\begin{align}\nonumber
    G[T_0(v)], \underset{s \:\text{times}}{\underbrace{G[T_1^{1}(v)], \dots, G[T_1^{s}(v)]}}, \dots, \underset{s \:\text{times}}{\underbrace{G[T_m^{1}(v)], \dots, G[T_m^{s}(v)]}}
\end{align}
as its corresponding sequence of tokens. Here, we can see another advantage of our proposed neighborhood sampling compared to \citet{chen2023nagphormer, ding2023recurrent}. While in NAGphormer~\citep{chen2023nagphormer} the sequence length of each node is limited by the diameter of the graph, our method can produce a long sequence of diverse subgraphs. 

\begin{theorem}\label{thm:neighborhood}
    With large enough $M, m,$ and $s > 0$, GMNs' neighborhood sampling is strictly more expressive than $k$-hop neighborhood sampling.
\end{theorem}

\noindent
\textcolor{c3}{\textbf{Structural/Positional Encoding}.}
To further augment our framework for Graph Mamba, we consider an optional step, when we inject structural and positional encodings to the initial features of nodes/edges. PE is meant to provide information about the position of a given node within the graph. Accordingly, two close nodes within a graph or subgraph are supposed to have close PE. SE, on the other hand, is meant to provide information about the structure of a subgraph. Following \citet{GPS}, we concatenate either eigenvectors of the graph Laplacian or Random-walk structural encodings to the nodes' feature, whenever PE/SE are needed: i.e., 
\begin{align}
    \mathbf{x}^{(\text{new})}_v = \mathbf{x}_v \: || \: p_v,
\end{align}
where $p_v$ is the corresponding positional encoding to $v$. For the sake of consistency, we use $\mathbf{x}_v$ instead of $\mathbf{x}^{\text{(new)}}_v$ throughout the paper.

\head{Neighborhood Encoding}
Given a node $v \in V$ and its sequence of tokens (subgraphs), we encode the subgraph via encoder $\phi(.)$. That is, we construct $\mathbf{x}_v^{1}, \mathbf{x}_v^{2}, \dots, \mathbf{x}_v^{ms-1}, \mathbf{x}_v^{ms} \in \R^{d}$ as follows:
\begin{align}\label{eq:token-encoding}
    \mathbf{x}_v^{\left((i-1)s + j\right)} = \phi\left(G[T^{j}_i(v)], \mathbf{X}_{T^{j}_i(v)} \right),
\end{align}
where $1 \leq i \leq m$ and $1 \leq j \leq s$. In practice, this encoder can be an MPNN, (e.g., Gated-GCN~\citep{bresson2017residual}), or RWF~\citep{nshoff2023walking} that encodes nodes with respect to a sampled set of walks into feature vectors with four parts: (1) node features, (2) edge features along the walk, and (3, 4) local structural information.

\head{Token Ordering}
By Equation~\ref{eq:token-encoding}, we can calculate the neighborhood embeddings for various sampled neighborhoods of a node and further construct a sequence to represent its neighborhood information, i.e., $\mathbf{x}_v^{1}, \mathbf{x}_v^{2}, \dots, \mathbf{x}_v^{ms-1}, \mathbf{x}_v^{ms}$. As discussed in \S\ref{sec:motivations}, adaption of sequence models like SSMs to graph-structured data requires an order on the tokens. To understand what constitutes a good ordering, we need to recall selection mechanism in Mamba~\citep{gu2023mamba} (we will discuss selection mechanism more formally in \S\ref{sec:bidirectional-mamba}). Mamba by making $\mathbf{B}, \mathbf{C},$ and $\boldsymbol{\Delta}$ as functions of input $x_t$ (see \S\ref{sec:ssm} for notations) lets the model filter irrelevant information and select important tokens in a recurrent manner, meaning that each token gets updated based on tokens that come before them in the sequence. Accordingly, earlier tokens have less information about the context of sequence, while later tokens have information about almost entire sequence. This leads us to order tokens based on either their needs of knowing information about other tokens or their importance to our task. 

\noindent
\textcolor{c1}{\underline{When $m \geq 1$}:} For the sake of simplicity first let $s=1$. In the case that $m \geq 1$, interestingly, our architecture design provides us with an implicitly ordered sequence. That is, given $v \in V$, the $i$-th token is a samples from $i$-hop neighborhood of node $v$, which is the subgraph of all $j$-hop neighborhoods, where $j \geq i$. This means, given a large enough $M$ (number of sampled random walks), our $T_{j}(v)$ has enough information about $T_i(v)$, not vice versa. To this end, we use the reverse of initial order, i.e., $\mathbf{x}_v^{m}, \mathbf{x}_v^{m - 1}, \dots, \mathbf{x}_v^{2}, \mathbf{x}_v^{1}$. Accordingly, inner subgraphs can also have information about the global structure. When $s \geq 2$, we use the same procedure as above, and reverse the initial order, i.e., $\mathbf{x}_v^{sm}, \mathbf{x}_v^{sm - 1}, \dots, \mathbf{x}_v^{2}, \mathbf{x}_v^{1}$. To make our model robust to the permutation of subgraphs with the same walk length $\hat{m}$, we randomly shuffle them. We will discuss the ordering in the case of $m = 0$ later.

\subsection{Bidirectional Mamba}\label{sec:bidirectional-mamba}
As discussed in \S\ref{sec:motivations}, SSMs are recurrent models and require ordered input, while graph-structured data does not have any order and needs permutation equivariant encoders. To this end, inspired by Vim in computer vision~\citep{zhu2024vision}, we modify Mamba architecture and use two recurrent scan modules to scan data in two different directions (i.e., forward and backward). Accordingly, given two tokens $t_i$ and $t_j$, where $i > j$ and indices show their initial order, in forward scan $t_i$ comes after $t_j$ and so has the information about $t_j$ (which can be flown into the hidden states or filtered by the selection mechanism). In backward pass $t_j$ comes after $t_i$ and so has the information about $t_i$. This architecture is particularly important when $m = 0$ (node tokenization), which we will discuss later. 

More formally, in forward pass module, let $\boldsymbol{\Phi}$ be the input sequence (e.g., given $v$, $\boldsymbol{\Phi}$ is a matrix whose rows are $\mathbf{x}_v^{sm}, \mathbf{x}_v^{sm - 1}, \dots,\mathbf{x}_v^{1}$, calculated in Equation~\ref{eq:token-encoding}), $\mathbf{A}$ be the relative positional encoding of tokens, we have:
\begin{align} \nonumber
    &\boldsymbol{\Phi}_{\text{input}} = \sigma\left(\texttt{Conv}\left( \mathbf{W}_{\text{input}} \: \texttt{LayerNorm}\left( \boldsymbol{\Phi}\right) \right)\right), \\ \nonumber
    &\mathbf{B} = \mathbf{W}_{\textbf{B}} \: \boldsymbol{\Phi}_{\text{input}}\:, \quad \mathbf{C} = \mathbf{W}_{\textbf{C}} \: \boldsymbol{\Phi}_{\text{input}}\:, \quad \boldsymbol{\Delta} = \texttt{Softplus}\left( \mathbf{W}_{\Delta} \: \boldsymbol{\Phi}_{\text{input}}\right),\\ \nonumber
    &\bar{\mathbf{A}} = \texttt{Discrete}_{\mathbf{A}}\left(\mathbf{A}, \boldsymbol{\Delta}  \right),\\ \nonumber
    &\bar{\mathbf{B}} = \texttt{Discrete}_{\mathbf{B}}\left(\mathbf{B}, \boldsymbol{\Delta}  \right),\\ \nonumber
    &\boldsymbol{y} = \texttt{SSM}_{\bar{\mathbf{A}}, \bar{\mathbf{B}}, \mathbf{C} }\left( \boldsymbol{\Phi}_{\text{input}} \right),\\ 
    &\boldsymbol{y}_{\text{forward}} = \mathbf{W}_{\text{forward}, 1}\left(\boldsymbol{y}  \odot  \sigma\left( \mathbf{W}_{\text{forward}, 2} \: \texttt{LayerNorm}\left( \boldsymbol{\Phi}  \right) \right) \right),
\end{align}
 where $\mathbf{W}, \mathbf{W}_{\mathbf{B}}, \mathbf{W}_{\mathbf{C}}, \mathbf{W}_{\boldsymbol{\Delta}}, \mathbf{W}_{\text{forward}, 1}$ and $\mathbf{W}_{\text{forward}, 2}$ are learnable parameters, $\sigma(.)$ is nonlinear function (e.g., \texttt{SiLU}), $\texttt{LayerNorm}(.)$ is layer normalization~\citep{layer-normalization}, $\texttt{SSM}(.)$ is the state space model discussed in Equations \ref{eq:ssm1} and \ref{eq:ssm2}, and $\texttt{Discrete}(.)$ is discretization process discussed in Equation~\ref{eq:ssm_disc}. We use the same architecture as above for the backward pass (with different weights) but instead we use $\boldsymbol{\Phi}_{\text{inverse}}$ as the input, which is a matrix whose rows are $\mathbf{x}_v^{1}, \mathbf{x}_v^{2}, \dots,\mathbf{x}_v^{sm}$. Let $\boldsymbol{y}_{\text{backward}}$ be the output of this backward module, we obtain the final encodings as 
 \begin{align}\label{eq:node-encodings}
     \boldsymbol{y}_{\text{output}} = \mathbf{W}_{\text{out}} \left( \boldsymbol{y}_{\text{forward}} + \boldsymbol{y}_{\text{backward}} \right).
 \end{align}
 In practice, we stack some layers of the bidirectional Mamba to achieve good performance. Note that due to our ordering mechanism, the last state of the output corresponds to the walk with length $\hat{m} = 0$, i.e., the node itself. Accordingly, the last state represents the updated node encoding.

\noindent
\textcolor{c3}{\textbf{Augmentation with MPNNs}.} We further use an optional MPNN module that simultaneously performs message-passing and augments the output of the bidirectional Mamba via its inductive bias. Particularly this module is very helpful when there are rich edge features and so an MPNN can help to take advantage of them. While in our empirical evaluation we show that this module is not necessary for the success of GMNs in several cases, it can be useful when we avoid complex PE/SE and strong inductive bias is needed.

\noindent
\textcolor{c1}{\textbf{How Does Selection Work on Subgraphs}?} 
As discussed earlier, the selection mechanism can be achieved by making $\mathbf{B}, \mathbf{C},$ and $\boldsymbol{\Delta}$ as the functions of the input data~\citep{gu2023mamba}. Accordingly, in recurrent scan, based on the input, the model can filter the irrelevant context. The selection mechanism in Equation~\ref{eq:node-encodings} is implemented by making $\mathbf{B}, \mathbf{C},$ and $\boldsymbol{\Delta}$ as functions of $\boldsymbol{\Phi}_{\text{input}}$, which is matrix of the encodings of neighborhoods. Therefore, as model scans the sampled subgraphs from the $i$-hop neighborhoods in descending order of $i$, it filters irrelevant neighborhoods to the context (last state), which is the node encoding.

\head{Last Layer(s) of Bidirectional Mamba}
To capture the long-range dependencies and to flow information across the nodes, we use the node encodings obtained from the last state of Equation~\ref{eq:node-encodings} as the input of the last layer(s) of bidirectional Mamba. Therefore, the recurrent scan of nodes (in both directions) can flow information across nodes. This design not only helps capturing long-range dependencies in the graph, but it also is a key to the flexibility of our framework to bridge node and subgraph tokenization.

\subsection{Tokenization When $m = 0$}
In this case, for each node $v \in V$ we only consider $v$ itself as its corresponding sequence of tokens. Based on our architecture, in this case, the first layers of bidirection Mamba become simple projection as the length of the sequence is one. However, the last layers, where we use node encodings as their input, treats nodes as tokens and become an architecture that use a sequential encoder (e.g., Mamba) with node tokenization. More specifically, in this special case of framework, the model is the adaption of GPS~\citep{GPS} framework, when we replace its Transformer with our bidirectional Mamba. 

This architecture design lets switching between node and subgraph tokenization using a single parameter $m$, making the choice of tokenization a tunable hyperparameter during training. Note that this flexibility comes more from our architecture rather than the method of tokenization. That is, in practice one can use only $0$-hop neighborhood in NAGphormer~\citep{chen2023nagphormer}, resulting in only considering the node itself. However, in this case, the architecture of NAGphormer becomes a stack of MLPs, resulting in poor performance. 

\head{Token Ordering} \textcolor{c1}{\underline{When $m = 0$}:} One remaining question is how one can order nodes when we use node tokenization. As discussed in \S\ref{sec:tokenization}, tokens need to be ordered based on either (1) their needs of knowing information about other tokens or (2) their importance to our task. When dealing with nodes and specifically when long-range dependencies matter, (1) becomes a must for all nodes. Our architecture overcomes this challenge by its bidirectional scan process. Therefore, we need to order nodes based on their importance. There are several metrics to measure the importance of nodes in a graph. For example, various centrality measures~\citep{latora2007measure, ruhnau2000eigenvector}, degree, $k$-core~\citep{k-core-first, FirmCore}, Personalized PageRank or PageRank~\citep{page1998pagerank}, etc. In our experiments, for the sake of efficiency and simplicity, we sort nodes based on their degree. 

\noindent
\textcolor{c1}{\textbf{How Does Selection Work on Nodes}?}
Similar to selection mechanism on subgraphs, the model based on the input data can filter irrelevant tokens (nodes) to the context (downstream tasks).

\subsection{Theoretical Analysis of GMNs}
In this section, we provide theoretical justification for the power of GMNs. More specifically, we first show that GMNs are universal approximator of any function on graphs. Next, we discuss that given proper PE and enough parameters, GMNs are more powerful than any WL isomorphism test, matching GTs (with the similar assumptions). Finally, we evaluate the expressive power of GMNs when they do not use any PE or MPNN and show that their expressive power is unbounded (might be incomparable).

\begin{theorem}[Universality]
    Let $1 \leq p < \infty$, and $\epsilon > 0$. For any continues function $f: [0, 1]^{d \times n} \rightarrow \R^{d \times n}$ that is permutation equivariant, there exists a GMN with positional encoding, $g_p$, such that $\ell^p(f, g) < \epsilon$ \footnote{$\ell^p(.)$ is the $p$-norm}. 
\end{theorem}

\begin{theorem}[Expressive Power w/ PE]
    Given the full set of eigenfunctions and enough parameters, GMNs can distinguish any pair of non-isomorphic graphs and are more powerful than any WL test.
\end{theorem}

We prove the above two theorems based on the recent work of  \citet{wang2023statespace}, where they prove that SSMs with layer-wise nonlinearity are universal approximators of any sequence-to-sequence function.   

\begin{theorem}[Expressive Power w/o PE and MPNN]\label{thm:unbounded}
    With enough parameters,
    for every $k \geq 1$ there are graphs that are distinguishable by GMNs, but not by $k$-WL test, showing that their expressive power is not bounded by any WL test.
\end{theorem}

We prove the above theorem based on the recent work of \citet{nshoff2023walking}, where they prove a similar theorem for CRaWl~\citep{nshoff2023walking}. Note that this theorem does not rely on the Mamba's power, and the expressive power comes from the choice of neighborhood sampling and encoding.

\section{Experiments}\label{sec:experiments}
In this section, we evaluate the performance of GMNs in long-range, small-scale, large-scale, and heterophilic benchmark datasets. We further discuss its memory efficiency and perform ablation study to validate the contribution of each architectural choice. The detailed statistics of datasets and additional experiments are available in the appendix.

\subsection{Experimental Setup}

\head{Dataset}
We use three most commonly used benchmark datasets with long-range, small-scale, large-scale, and heterophilic properties. For long-range datasets, we use Longe Range Graph Benchmark (LRGB) dataset~\citep{long-range-data}. For small and large-scale datasets, we use GNN benchmark~\citep{dwivedi2023benchmarking}. To evaluate the GMNs on heterophilic graphs, we use four heterophilic datasets from the work of \citet{platonov2023a}. Finally, we use a large dataset from Open Graph Benchmark~\citep{hu2020open}. We evaluate the performance of GMNs on various graph learning tasks (e.g., graph classification, regression, node classification and link classification). Also, for each datasets we use the propose metrics in the original benchmark and report the metric across multiple runs, ensuring the robustness. We discuss datasets, their statistics and their tasks in Appendix~\ref{app:data}. 

\head{Baselines}
We compare our GMNs with (1) MPNNs, e.g., GCN~\citep{kipf2016semi}, GIN~\citep{xu2018how}, and Gated-GCN~\citep{bresson2017residual}, (2) Random walk based method CRaWl~\citep{nshoff2023walking}, (3) state-of-the-art GTs, e.g., SAN~\citep{kreuzer2021rethinking}, NAGphormer~\citep{chen2023nagphormer}, Graph ViT~\citep{graph-mlpmixer}, two variants of GPS~\citep{GPS}, GOAT~\citep{kong2023goat}, and Exphormer~\citep{shirzad2023exphormer}, and (4) our baselines (i) GPS + Mamba: when we replace the transformer module in GPS with bidirectional Mamba. (ii) GMN-: when we do not use PE/SE and MPNN. The details of baselines are in Appendix~\ref{app:experimental-setup}.

\begin{table*}
    \centering
    \caption{Benchmark on Long-Range Graph Datasets~\citep{long-range-data}. Highlighted are the top \first{first}, \second{second}, and \third{third} results.}
    \resizebox{0.67\linewidth}{!}{
    \begin{tabular}{l  c  c  c  c }
    \toprule
        \multirow{2}{*}{\textbf{Model}} & \textbf{COCO-SP}  & \textbf{PascalVOC-SP} & \textbf{Peptides-Func}  & \textbf{Peptides-Struct} \\
         & F1 score $\uparrow$&  F1 score $\uparrow$ & AP $\uparrow$ & MAE $\downarrow$ \\
         \midrule
         \midrule
         GCN & $0.0841_{\pm 0.0010}$ & $0.1268_{\pm 0.0060}$ & $0.5930_{\pm 0.0023}$ & $0.3496_{\pm 0.0013}$ \\
         GIN & $0.1339_{\pm 0.0044}$ & $0.1265_{\pm 0.0076}$ & $0.5498_{\pm 0.0079}$ & $0.3547_{\pm 0.0045}$ \\
         Gated-GCN & $0.2641_{\pm 0.0045}$ & $0.2873_{\pm 0.0219}$ & $0.5864_{\pm 0.0077}$ & $0.3420_{\pm 0.0013}$ \\
         \midrule
         CRaWl & $0.3219_{\pm 0.00106}$ & {$0.4088_{\pm 0.0079}$} & \second{$0.6963_{\pm 0.0079}$} & {$0.2506_{\pm 0.0022}$} \\
         \midrule
         SAN+LapPE & $0.2592_{\pm 0.0158}$ & $0.3230_{\pm 0.0039}$ & $0.6384_{\pm 0.0121}$ & $0.2683_{\pm 0.0043}$ \\
         NAGphormer & $0.3458_{\pm 0.0070}$ & $0.4006_{\pm 0.0061}$ & - & - \\
         Graph ViT & - & - & {$0.6855_{\pm 0.0049}$} & \first{$0.2468_{\pm 0.0015}$}\\
         GPS & \third{$0.3774_{\pm 0.0150}$} & $0.3689_{\pm 0.0131}$ & $0.6575_{\pm 0.0049}$ & $0.2510_{\pm 0.0015}$\\
         GPS (BigBird) & $0.2622_{\pm 0.0008}$ & $0.2762_{\pm 0.0069}$ & $0.5854_{\pm 0.0079}$ & $0.2842_{\pm 0.0130}$\\
         Exphormer & $0.3430_{\pm 0.0108}$ & $0.3975_{\pm 0.0037}$ & $0.6527_{\pm 0.0043}$ &  \third{$0.2481_{\pm 0.0007}$} \\
         \midrule
         GPS + Mamba & \second{$0.3895_{\pm 0.0125}$} & \second{$0.4180_{\pm 0.012}$} & $0.6624_{\pm 0.0079}$ & $0.2518_{\pm 0.0012}$ \\
         GMN- & $0.3618_{\pm 0.0053}$ & \third{$0.4169_{\pm 0.0103}$} & \third{$0.6860_{\pm 0.0012}$} & $0.2522_{\pm 0.0035}$ \\
         GMN & \first{$0.3974_{\pm 0.0101}$} & \first{$0.4393_{\pm 0.0112}$} & \first{$0.7071_{\pm 0.0083}$} & \second{$0.2473_{\pm 0.0025}$} \\
    \toprule
    \end{tabular}
    \label{tab:LRGD}
    }
\end{table*}

\subsection{Long Range Graph Benchmark}
Table~\ref{tab:LRGD} reports the results of GMNs and baselines on long-range graph benchmark. GMNs consistently outperform baselines in all datasets that requires long-range dependencies between nodes. The reason for this superior performance is three folds: (1) GMNs based on our design use long sequence of tokens to learn node encodings and then use another selection mechanism to filter irrelevant nodes. The provided long sequence of tokens enables GMNs to learn long-range dependencies, without facing scalability or over-squashing issues. (2) GMNs using their selection mechanism are capable of filtering the neighborhood around each node. Accordingly, only informative information flows into hidden states. (3) The random-walk based neighborhood sampling allow GMNs to have diverse samples of neighborhoods, while capturing the hierarchical nature of $k$-hop neighborhoods. Also, it is notable that GMN consistently outperforms our baseline GPS + Mamba, which shows the importance of paying attention to the new challenges. That is, replacing the transformer module with Mamba, while improves the performance, cannot fully take advantage of the Mamba traits. Interestingly, GMN-, a variant of GMNs without Transformer, MPNN, and PE/SE that we use to evaluate the importance of these elements in achieving good performance, can achieve competitive performance with other complex methods, showing that while Transformers, complex message-passing, and SE/PE are sufficient for good performance in practice, neither is necessary.

\subsection{Comparison on GNN Benchmark}
We further evaluate the performance of GMNs in small and large datasets from the GNN benchmark. The results of GMNs and baseline performance are reported in Table~\ref{tab:GNN}. GMN and Exphormer achieve competitive performance each outperforms the other two times. On the other hand again, GMN consistently outperforms GPS + Mamba baseline, showing the importance of designing a new framework for GMNs rather then using existing frameworks of GTs.  

\begin{table}
    \centering
    \caption{Benchmark on GNN Benchmark Datasets~\citep{dwivedi2023benchmarking}. Highlighted are the top \first{first}, \second{second}, and \third{third} results.}
    \resizebox{0.65\linewidth}{!}{
    \begin{tabular}{l  c  c   c  c }
    \toprule
        \multirow{2}{*}{\textbf{Model}} & \textbf{MNIST}  & \textbf{CIFAR10} & \textbf{PATTERN}  & \textbf{MalNet-Tiny} \\
         & Accuracy $\uparrow$& Accuracy $\uparrow$ & Accuracy $\uparrow$ & Accuracy $\uparrow$\\
         \midrule
         \midrule
         GCN & $0.9071_{\pm 0.0021}$ & $0.5571_{\pm 0.0038}$ & $0.7189_{\pm 0.0033}$ & $0.8100_{\pm 0.0000}$ \\
         GIN & $0.9649_{\pm 0.0025}$ & $0.5526_{\pm 0.0152}$ & $0.8539_{\pm 0.0013}$ & $0.8898_{\pm 0.0055}$ \\
         Gated-GCN & $0.9734_{\pm 0.0014}$ & $0.6731_{\pm 0.0031}$ & $0.8557_{\pm 0.0008}$ & $0.9223_{\pm 0.0065}$ \\
         \midrule
         CRaWl & $0.9794_{\pm 0.050}$ & $0.6901_{\pm 0.0259}$ & - & - \\
         \midrule
         NAGphormer & - & - & $0.8644_{\pm 0.0003}$ & - \\
         GPS & $0.9811_{\pm 0.0011}$ & $0.7226_{\pm 0.0031}$ & \third{$0.8664_{\pm 0.0011}$} & $0.9298_{\pm 0.0047}$\\
         GPS (BigBird) & $0.9817_{\pm 0.0001}$ & $0.7048_{\pm 0.0010}$ & $0.8600_{\pm 0.0014}$ & $0.9234_{\pm 0.0034}$\\
         Exphormer & \first{$0.9855_{\pm 0.0003}$} & \second{$0.7469_{\pm 0.0013}$} & \second{$0.8670_{\pm 0.0003}$} &  \second{$0.9402_{\pm 0.0020}$}\\
         \midrule
         GPS + Mamba & \third{$0.9821_{\pm 0.0004}$} & \third{$0.7341_{\pm 0.0015}$} & $0.8660_{\pm 0.0007}$ & \third{$0.9311_{\pm 0.0042}$} \\
         GMN & \second{$0.9839_{\pm 0.0018}$} & \first{$0.7576_{\pm 0.0042}$} & \first{$0.8714_{\pm 0.0012}$} & \first{$0.9415_{\pm 0.0020}$} \\
    \toprule
    \end{tabular}
    \label{tab:GNN}
    }
\end{table}

\vspace{-2ex}
\subsection{Heterophilic Datasets}
To evaluate the performance of GMNs on the heterophilc data as well as evaluating their robustness to over-squashing and over-smoothing, we compare their performance with the state-of-the-art baselines and report the results in Table~\ref{tab:heterophilic}. Our GMN outperforms baselines in 3 out of 4 datasets and achieve the second best result in the remaining dataset. These results show that the selection mechanism in GMN can effectively filter irrelevant information and also consider long-range dependencies in heterophilic datasets.

\begin{table}
    \centering
    \caption{Benchmark on heterophilic datasets~\citep{platonov2023a}. Highlighted are the top \first{first}, \second{second}, and \third{third} results.}
    \resizebox{0.70\linewidth}{!}{
    \begin{tabular}{l  c  c  c  c  c }
    \toprule
        \multirow{2}{*}{\textbf{Model}} & \textbf{Roman-empire}  & \textbf{Amazon-ratings} & \textbf{Minesweeper}  & \textbf{Tolokers}\\
         & Accuracy $\uparrow$& Accuracy $\uparrow$ & ROC AUC $\uparrow$ & ROC AUC $\uparrow$\\
         \midrule
         \midrule
         GCN & $0.7369_{\pm 0.0074}$ & $0.4870_{\pm 0.0063}$ & $0.8975_{\pm 0.0052}$ & $0.8364_{\pm 0.0067}$ \\
         Gated-GCN & $0.7446_{\pm 0.0054}$ & $0.4300_{\pm 0.0032}$ & $0.8754_{\pm 0.0122}$ & $0.7731_{\pm 0.0114}$ \\
         \midrule
         NAGphormer & $0.7434_{\pm 0.0077}$ & $0.5126_{\pm 0.0072}$ & $0.8419_{\pm 0.0066}$ & $0.7832_{\pm 0.0095}$ \\
         GPS & $0.8200_{\pm 0.0061}$ & \third{$0.5310_{\pm 0.0042}$} & \third{$0.9063_{\pm 0.0067}$} & \third{$0.8371_{\pm 0.0048}$}\\
         Exphormer & \first{$0.8903_{\pm 0.0037}$} & \second{$0.5351_{\pm 0.0046}$} & \second{$0.9074_{\pm 0.0053}$} & \second{$0.8377_{\pm 0.0078}$}\\
         GOAT & $0.7159_{\pm 0.0125}$ & $0.4461_{\pm 0.0050}$ & $0.8109_{\pm 0.0102}$ & $0.8311_{\pm 0.0104}$ \\
         \midrule
         GPS + Mamba & \third{$0.8310_{\pm 0.0028}$} & $0.4513_{\pm 0.0097}$ & $0.8993_{\pm 0.0054}$ & $0.8370_{\pm 0.0105}$ \\
         GMN & \second{$0.8769_{\pm 0.0050}$} & \first{$0.5407_{\pm 0.0031}$} & \first{$0.9101_{\pm 0.0023}$} & \first{$0.8452_{\pm 0.0021}$} \\
    \toprule
    \end{tabular}
    \label{tab:heterophilic}
    }
\end{table}

\subsection{Ablation Study} \label{sec:ablation-study}
To evaluate the contribution of each component of GMNs in its performance, we perform ablation study. Table~\ref{tab:my_label} reports the results. The first row, reports the performance of GMNs with its full architecture. Then in each row, we modify one the elements while keeping the other unchanged: Row 2 remove the bidirectional Mamba and use a simple Mamba. Row 3 remove the MPNN and Row 4 use PPR ordering. Finally the last row remove PE. Results show that all the elements of GMN contributes to its performance with most contribution from bidirection Mamba.

\begin{table}
    \centering
    \caption{Ablation study on GMN architecture.}
    \resizebox{0.7\linewidth}{!}{
    \begin{tabular}{l  c  c  c  c }
    \toprule
        \multirow{2}{*}{\textbf{Model}} & \textbf{Roman-empire}  & \textbf{Amazon-ratings} & \textbf{Minesweeper} \\
         & Accuracy $\uparrow$& Accuracy $\uparrow$ & ROC AUC $\uparrow$\\
         \midrule
         \midrule
         GMN & \first{$0.8769_{\pm 0.0050}$} & \first{$0.5407_{\pm 0.0031}$} & \first{$0.9101_{\pm 0.0023}$}   \\
         w/o bidirectional Mamba & $0.8327_{\pm 0.0062}$ & $0.5016_{\pm 0.0045}$ & $0.8597_{\pm 0.0028}$ \\
         w/o MPNN & \second{$0.8620_{\pm 0.0043}$} & \second{$0.5312_{\pm 0.0044}$} & {$0.8983_{\pm 0.0031}$} \\
         PPR ordering & \third{$0.8612_{\pm 0.0019}$} & \third{$0.5299_{\pm 0.0037}$} & \third{$0.8991_{\pm 0.0021}$} \\
         w/o PE & $0.8591_{\pm 0.0054}$ & $0.5308_{\pm 0.0026}$ & \second{$0.9011_{\pm 0.0025}$} \\
    \toprule
    \end{tabular}
    \label{tab:my_label}
    }
\end{table}

\subsection{Efficiency}
As we discussed earlier, one of the main advantages of our model is its efficiency and memory usage. We evaluate this claim on OGBN-Arxiv~\citep{hu2020open} and MalNet-Tiny~\citep{dwivedi2023benchmarking} datasets and report the results in Figure~\ref{tab:efficiency}. Our variants of GMNs are the most efficienct methods while achieving the best performance. To show the trend of scalability, we use MalNet-Tiny and plot the memory usage of GPS and GMN in Figure~\ref{fig:memory}. While GPS, as a graph transformer framework, requires high computational cost (GPU memory usage), GMNs's memory scales linearly with respect to the input~size.

\begin{figure}
  \begin{minipage}[b]{.55\linewidth}
    \centering
    \caption{$\text{Efficiency evaluation and accuracy of GMNs and baselines on OBGN-Arxiv and}$ $\text{MalNet-Tiny. Highlighted are the top \first{first}, \second{second}, and \third{third} results. OOM: Out of Memory.}$} \label{tab:efficiency}
    \resizebox{1.2\linewidth}{!}{
    \begin{tabular}{c c c c c c c c}
    \toprule
        \multirow{2}{*}{\textbf{Method}} & \multirow{2}{*}{Gated-GCN} &  \multirow{2}{*}{GPS} & \multirow{2}{*}{NAGphormer} & \multirow{2}{*}{Exphormer$^{\dagger}$} & \multirow{2}{*}{GOAT} & \multicolumn{2}{c}{Ours}\\
        \cmidrule(lr){7-8}
        & & & & & & GPS+Mamba & $\:\:\: \:$GMN $\: \:\:$\\
        \midrule
        \multicolumn{8}{c}{OGBN-Arxiv}\\
        \midrule
         Training/Epoch (s) & \first{0.68} & OOM & 5.06 & 1.97 & 13.09 & \second{1.18} & \third{1.30}\\
         Memory (GB) & 11.09 & OOM & \third{6.24} & 36.18 & 8.41 & \second{5.02} & \first{3.85}\\
         Accuracy & 0.7141 & OOM & 0.7013 & \third{0.7228} & 0.7196 & \second{0.7239} & \first{0.7248}\\
         \midrule
         \multicolumn{8}{c}{MalNet-Tiny}\\
         \midrule
         Training/Epoch (s) & \first{10.3} & 148.99 & - & 57.24 & - & \second{36.07} & \third{41.00}\\
         Accuracy & 0.9223 & \third{0.9234} & - & 0.9224 & - & \second{0.9311} & \first{0.9415}\\
    \toprule
        \multicolumn{8}{l}{$^{\dagger}$ We follow the original paper~\citep{shirzad2023exphormer} and use one virtual node in efficiency evaluation.}\\
    \end{tabular}
    }
  \end{minipage} \hfill
  \begin{minipage}{.3\linewidth}
    \centering
    \vspace{-7ex}
    \includegraphics[width=0.99\linewidth]{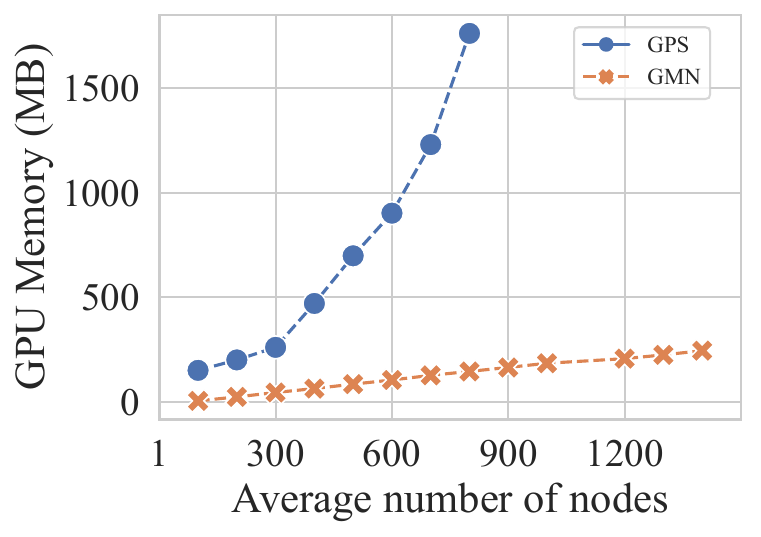}
    \vspace{-4ex}
    \caption{Memory of GPS and GMN on MalNet-Tiny dataset.}\label{fig:memory}
  \end{minipage}
\end{figure}

\section{Conclusion}
In this paper, we present Graph Mamba Networks (GMNs) as a new class of graph learning based on State Space Model. We discuss and categorize the new challenges when adapting SSMs to graph-structured data, and present four required and one optional steps to design GMNs, where we choose (1) Neighborhood Tokenization, (2) Token Ordering, (3) Architecture of Bidirectional Selective SSM Encoder, (4) Local Encoding, and dispensable (5) PE and SE. We further provide theoretical justification for the power of GMNs and conduct several experiments to empirically evaluate their performance.

\section*{Potential Broader Impact}
This paper presents work whose goal is to advance the field of Machine Learning. There are many potential societal consequences of our work, none of which we feel must be specifically highlighted here.

\newpage
\bibliography{main}
\bibliographystyle{icml2024}

\newpage
\appendix
\onecolumn

\section{Details of Datasets}\label{app:data}
The statistics of all the datasets are reported in Table~\ref{tab:dataset}. For additional details about the datasets, we refer to the Long-range graph benchmark~\citep{long-range-data}, GNN Benchmark~\citep{dwivedi2023benchmarking}, Heterophilic Benchmark \citep{platonov2023a}, and Open Graph Benchmark \citep{hu2020open}.
\begin{table}
    \centering
    \caption{Dataset Statistics.}
    \resizebox{0.90\linewidth}{!}{
    \begin{tabular}{l c  c  c  c  c  c c c c}
    \toprule
        \multicolumn{1}{c}{\multirow{2}{*}{Dataset}} & \multirow{2}{*}{\#Graphs} & \multirow{2}{*}{Average \#Nodes} & \multirow{2}{*}{Average \#Edges} & \multirow{2}{*}{\#Class} & \multicolumn{2}{c}{Setup} &  \multirow{2}{*}{Metric}\\
        \cmidrule(lr){6-7}
        & & & &  & Input Level & Task \\
        \midrule
        \multicolumn{8}{c}{Long-range Graph Benchmark~\citep{long-range-data}} \\
        \midrule
         \multirow{1}{*}{COCO-SP} & \multirow{1}{*}{123,286} & \multirow{1}{*}{476.9} & \multirow{1}{*}{2693.7}& 81 & \multirow{1}{*}{Node} & \multirow{1}{*}{Classification} & F1 score \\
        \multirow{1}{*}{PascalVOC-SP}& \multirow{1}{*}{11,355} & \multirow{1}{*}{479.4} & \multirow{1}{*}{2710.5} & 21 &\multirow{1}{*}{Node} &\multirow{1}{*}{Classification} & F1 score \\
        \multirow{1}{*}{Peptides-Func}& \multirow{1}{*}{15,535} & \multirow{1}{*}{150.9} & \multirow{1}{*}{307.3} & 10 &\multirow{1}{*}{Graph} &\multirow{1}{*}{Classification} & Average Precision\\
        \multirow{1}{*}{Peptides-Struct}& \multirow{1}{*}{15,535} & \multirow{1}{*}{150.9} & \multirow{1}{*}{307.3} & 11 (regression) &\multirow{1}{*}{Graph} &\multirow{1}{*}{Regression} & Mean Absolute Error\\
        \midrule
        \multicolumn{8}{c}{GNN Benchmark~\citep{dwivedi2023benchmarking}}\\
        \midrule
        \multirow{1}{*}{MNIST} & 70,000 & 70.6 & 564.5 & 10 & Graph & Classification & Accuracy \\
         \multirow{1}{*}{CIFAR10} &\multirow{1}{*}{60,000} & \multirow{1}{*}{117.6} & \multirow{1}{*}{941.1} & 10 &\multirow{1}{*}{Graph} &\multirow{1}{*}{Classification} & Accuracy \\
         \multirow{1}{*}{Pattern} & \multirow{1}{*}{14,000} & \multirow{1}{*}{118.9} & \multirow{1}{*}{3,039.3} & 2 & \multirow{1}{*}{Node} &\multirow{1}{*}{Classification} & Accuracy \\
         \multirow{1}{*}{MalNet-Tiny} & \multirow{1}{*}{5,000} & \multirow{1}{*}{1,410.3} & \multirow{1}{*}{2,859.9} & 5 &\multirow{1}{*}{Graph} &\multirow{1}{*}{Classification}& Accuracy \\ 
        \midrule
        \multicolumn{8}{c}{Heterophilic Benchmark \citep{platonov2023a}}\\
        \midrule
        \multirow{1}{*}{Roman-empire} & 1 & 22,662 & 32,927 &  18 & Node & Classification & Accuracy\\
        \multirow{1}{*}{Amazon-ratings} &\multirow{1}{*}{1} & \multirow{1}{*}{24,492} & \multirow{1}{*}{93,050} & \multirow{1}{*}{5} &\multirow{1}{*}{Node} & Classification & Accuracy  \\
         \multirow{1}{*}{Minesweeper} & 1 & \multirow{1}{*}{10,000} & \multirow{1}{*}{39,402} & \multirow{1}{*}{2} &\multirow{1}{*}{Node} & Classification & ROC AUC\\
         \multirow{1}{*}{Tolokers}&  \multirow{1}{*}{1} & \multirow{1}{*}{11,758} & \multirow{1}{*}{519,000} & \multirow{1}{*}{2} &\multirow{1}{*}{Node} & Classification & ROC AUC\\ 
         \midrule
         \multicolumn{8}{c}{Very Large Dataset~\citep{hu2020open}}\\
        \midrule
        OGBN-Arxiv &  \multirow{1}{*}{1} & \multirow{1}{*}{169,343} & \multirow{1}{*}{1,166,243} & \multirow{1}{*}{40} &\multirow{1}{*}{Node} & Classification & Accuracy\\
    \toprule
    \end{tabular}
    \label{tab:dataset}
    }
\end{table}

\begin{table}
    \centering
    \caption{Search space of hyperparameters for each dataset$^{\dagger}$.}
    \resizebox{0.9\linewidth}{!}{
    \begin{tabular}{l c  c  c   c c }
    \toprule
        \multicolumn{1}{c}{Dataset} & $M$ & $s$ & \#Layers &  Max \# Epochs & Learning Rate\\
        \midrule
        \multicolumn{6}{c}{Long-range Graph Benchmark~\citep{long-range-data}} \\
        \midrule
         \multirow{1}{*}{COCO-SP} & $\{1, 2, 4, 8, 16, 32 \}$ & $\{0, 1, 2, 4, 8, 16 \}$ & $\{4, 5\}$ &300 & 0.001\\
        \multirow{1}{*}{PascalVOC-SP}& $\{1, 2, 4, 8, 16, 32 \}$  & $\{0, 1, 2, 4, 8, 16 \}$ & $\{4, 5\}$  &300 & 0.001 \\
        \multirow{1}{*}{Peptides-Func}& $\{1, 2, 4, 8, 16, 32 \}$ & $\{0, 1, 2, 4, 8, 16 \}$ & $\{4, 5\}$  &300 & 0.001 \\
        \multirow{1}{*}{Peptides-Struct}& $\{1, 2, 4, 8, 16, 32 \}$ & $\{0, 1, 2, 4, 8, 16 \}$ &  $\{4, 5\}$ &300 & 0.001 \\
        \midrule
        \multicolumn{6}{c}{GNN Benchmark~\citep{dwivedi2023benchmarking}}\\
        \midrule
        \multirow{1}{*}{MNIST} & $\{1, 2, 4, 8, 16, 32 \}$ & $\{0, 1, 2, 4, 8, 16 \}$ & $\{3, 4, 6\}$ & 300  & 0.001  \\
         \multirow{1}{*}{CIFAR10} & $\{1, 2, 4, 8, 16, 32 \}$ & $\{0, 1, 2, 4, 8, 16 \}$ & $\{3, 4, 6\}$  &300 & 0.001\\
         \multirow{1}{*}{Pattern} & $\{1, 2, 4, 8, 16, 32 \}$ & $\{0, 1, 2, 4, 8, 16 \}$ & $\{3, 4, 6\}$  &300 & 0.001  \\
         \multirow{1}{*}{MalNet-Tiny} & $\{1, 2, 4, 8, 16, 32 \}$ & $\{0, 1, 2, 4, 8, 16 \}$  & $\{3, 4, 6\}$&300 & 0.001 \\ 
        \midrule
        \multicolumn{6}{c}{Heterophilic Benchmark \citep{platonov2023a}}\\
        \midrule
        \multirow{1}{*}{Roman-empire} & $\{1, 2, 4, 8, 16, 32 \}$ & $\{0, 1, 2, 4, 8, 16 \}$  & $\{3, 4, 6\}$& 300 & 0.001 \\
        \multirow{1}{*}{Amazon-ratings} & $\{1, 2, 4, 8, 16, 32 \}$ & $\{0, 1, 2, 4, 8, 16 \}$  & $\{3, 4, 6\}$& 300  & 0.001  \\
         \multirow{1}{*}{Minesweeper} & $\{1, 2, 4, 8, 16, 32 \}$ & $\{0, 1, 2, 4, 8, 16 \}$  & $\{3, 4, 6\}$& 300 & 0.001 \\
         \multirow{1}{*}{Tolokers}& $\{1, 2, 4, 8, 16, 32 \}$  & $\{0, 1, 2, 4, 8, 16 \}$ & $\{3, 4, 6\}$ &300 & 0.001\\ 
         \midrule
         \multicolumn{6}{c}{Very Large Dataset~\citep{hu2020open}}\\
        \midrule
        OGBN-Arxiv & $\{1, 2, 4, 8, 16, 32 \}$ & $\{0, 1, 2, 4, 8, 16 \}$ & $\{3, 4, 6\}$& 300 & 0.001 \\
    \toprule
    \multicolumn{6}{l}{$^{\dagger}$ This space is not fully searched and preliminary results are reported based on its subspace. We will update the results accordingly. }
    \end{tabular}
    \label{tab:hyperparameter}
    }
\end{table}

\section{Experimental Setup}\label{app:experimental-setup}

\subsection{Hyperparameters}
We use grid search to tune hyperparameters and the search space is reported in Table~\ref{tab:hyperparameter}. Following previous studies, we use the same split of traning/test/validation as \citep{GPS}. We report the results over the 4 random seeds. Also, for the baselines' results (in Tables \ref{tab:LRGD}, \ref{tab:GNN}, and \ref{tab:heterophilic}), we have re-used and reported the benchmark results in the work by \citet{shirzad2023exphormer, deng2024polynormer, tonshoff2023did} and \citet{GMB}.\footnote{In the previous version of this preprint (Feb 13, 2024), the reported results of Exphormer~\citep{shirzad2023exphormer} and GPS~\citep{GPS} came from the results of \citet{GMB}.}.

\begin{algorithm}[t]
    \small
    \caption{Graph Mamba Networks (with one layer)}
    \label{alg:GMN}
    \begin{algorithmic}[1]
        \Require{A graph $G = (V, E)$, input features $\mathbf{X} \in \R^{n \times d}$, ordered array of nodes $V = \{v_1, \dots, v_n\}$, and hyperparameters $M, m$, and $s$. \hspace*{-4ex}\textcolor{dark2green}{\textbf{Optional:}  Matrix $\mathbf{P}$, whose rows are positional/structural encodings correspond to nodes, and/or a MPNN model $\Psi(.)$.}}
        \Ensure{The updated node encodings $\mathbf{X}_{\text{new}}$.}
        \For{$v \in V$} \Comment{\third{This block can be done before the training.}}
            \For{$\hat{m} = 0, \dots, m$}
                \For{$\hat{s} = 1, \dots, s$}
                    \State $T_{\hat{m}}^{\hat{s}}(v) \leftarrow \emptyset$;
                    \For{$\hat{M} = 1, \dots, M$}
                        \State $\texttt{W} \leftarrow$ Sample a random walk with length $\hat{m}$ starting from $v$;
                        \State $T_{\hat{m}}^{\hat{s}}(v) \leftarrow T_{\hat{m}}^{\hat{s}}(v) \cup \{ u | u \in \texttt{W}\}$;
                    \EndFor
                \EndFor
            \EndFor
        \EndFor
        \State
        \State
        \Comment{\third{Start the training:}$\qquad\qquad\qquad\qquad\qquad\qquad\qquad\qquad\qquad\qquad\qquad\qquad\qquad\qquad\qquad\qquad\qquad\qquad\qquad\qquad\qquad\qquad$}
        \State Initialize all learnable parameters;
        \For{$v \in V$}
            \For{$ j = 1, \dots, s$}
                \For{$i = 1, \dots, m$}
                    \State $\mathbf{x}^{(i - 1)s + j}_v \leftarrow \phi\left(G[T^{j}_i(v)], \mathbf{X}_{T^{j}_i(v)} \first{$\: || \:\mathbf{P}_{T^{j}_i(v)}$} \right)$;
                \EndFor
            \EndFor
            $\boldsymbol{\Phi}_v \leftarrow \mathbin\Vert_{i = 1}^{sm} \: \mathbf{x}^{i}_{v}$; \Comment{\third{$\boldsymbol{\Phi}_v$ is a matrix whose rows are $\mathbf{x}^{i}_{v}$.}}
            \State $\boldsymbol{y}_{\text{output}}(v) \leftarrow \texttt{BiMamba}\left( \boldsymbol{\Phi}_v \right)$; \Comment{\second{Using Algorithm~\ref{alg:BiMamba}.}}
        \EndFor
        \State
        \Comment{\third{Each node is a token:}$\:\:\:\qquad\qquad\qquad\qquad\qquad\qquad\qquad\qquad\qquad\qquad\qquad\qquad\qquad\qquad\qquad\qquad\qquad\qquad\qquad\qquad\qquad$}
        \State $\boldsymbol{Y} \leftarrow \mathbin\Vert_{i = 1}^{sm} \:\boldsymbol{y}_{\text{output}}(v)$; \Comment{\third{$\boldsymbol{y}$ is a matrix whose rows are $\boldsymbol{y}_{\text{output}}(v)$.}}
        \State $\boldsymbol{Y}_{\text{output}} \leftarrow \texttt{BiMamba}\left( \boldsymbol{Y} \right) + \first{$\Psi\left( G, \mathbf{X} \Vert \mathbf{P} \right)$}$;
    \end{algorithmic}
\end{algorithm}

\begin{algorithm}[t]
    \small
    \caption{Bidirectional Mamba}
    \label{alg:BiMamba}
    \begin{algorithmic}[1]
        \Require{A sequence $\boldsymbol{\Phi}$ (Ordered matrix, where each row is a token).}
        \Ensure{The updated sequence encodings $\boldsymbol{\Phi}$.}
        \State
        \Comment{\third{Forward Scan:}$\qquad\qquad\qquad\qquad\qquad\qquad\:\:\quad\qquad\qquad\qquad\qquad\qquad\qquad\qquad\qquad\qquad\qquad\qquad\qquad\qquad\qquad\qquad\qquad$}
        \State $\boldsymbol{\Phi}_{\text{f}} = \sigma\left(\texttt{Conv}\left( \mathbf{W}_{\text{input}, f} \: \texttt{LayerNorm}\left( \boldsymbol{\Phi}\right) \right)\right)$;
        \State $\mathbf{B}_{\text{f}} = \mathbf{W}_{\textbf{B}_{\text{f}}} \: \boldsymbol{\Phi}_{\text{f}}$;
        \State $\mathbf{C}_{\text{f}} = \mathbf{W}_{\textbf{C}_{\text{f}}} \: \boldsymbol{\Phi}_{\text{f}}$;
        \State $\boldsymbol{\Delta}_{\text{f}} = \texttt{Softplus}\left( \mathbf{W}_{\Delta_{\text{f}}} \: \boldsymbol{\Phi}_{\text{f}}\right)$;
        \State $\bar{\mathbf{A}} = \texttt{Discrete}_{\mathbf{A}}\left(\mathbf{A}, \boldsymbol{\Delta}  \right)$;
        \State $\bar{\mathbf{B}}_{\text{f}} = \texttt{Discrete}_{\mathbf{B}_{\text{f}}}\left(\mathbf{B}_{\text{f}}, \boldsymbol{\Delta}  \right)$;
        \State $\boldsymbol{y}_{\text{f}} = \texttt{SSM}_{\bar{\mathbf{A}}, \bar{\mathbf{B}}_{\text{f}}, \mathbf{C}_{\text{f}} }\left( \boldsymbol{\Phi}_{\text{f}} \right)$;
        \State $\boldsymbol{Y}_{\text{f}} = \mathbf{W}_{\text{f}, 1}\left(\boldsymbol{y}_{\text{f}}  \odot  \sigma\left( \mathbf{W}_{\text{f}, 2} \: \texttt{LayerNorm}\left( \boldsymbol{\Phi}  \right) \right) \right)$;
        \State \Comment{\third{Backward Scan:}$\qquad\qquad\qquad\qquad\qquad\qquad\:\:\:\qquad\qquad\qquad\qquad\qquad\qquad\qquad\qquad\qquad\qquad\qquad\qquad\qquad\qquad\qquad\qquad$}

        \State $ \boldsymbol{\Phi} \leftarrow \texttt{Reverse-rows}\left( \boldsymbol{\Phi} \right)$; \Comment{\second{Reverse the order of rows in the matrix.}}
        \State $\boldsymbol{\Phi}_{\text{b}} = \sigma\left(\texttt{Conv}\left( \mathbf{W}_{\text{input, b}} \: \texttt{LayerNorm}\left( \boldsymbol{\Phi}\right) \right)\right)$;
        \State $\mathbf{B}_{\text{b}} = \mathbf{W}_{\textbf{B}_{\text{b}}} \: \boldsymbol{\Phi}_{\text{b}}$;
        \State $\mathbf{C}_{\text{b}} = \mathbf{W}_{\textbf{C}_{\text{b}}} \: \boldsymbol{\Phi}_{\text{b}}$;
        \State $\boldsymbol{\Delta}_{\text{b}} = \texttt{Softplus}\left( \mathbf{W}_{\Delta_{\text{b}}} \: \boldsymbol{\Phi}_{\text{b}}\right)$;
        \State $\bar{\mathbf{A}} = \texttt{Discrete}_{\mathbf{A}}\left(\mathbf{A}, \boldsymbol{\Delta}  \right)$;
        \State $\bar{\mathbf{B}}_{\text{b}} = \texttt{Discrete}_{\mathbf{B}_{\text{b}}}\left(\mathbf{B}_{\text{b}}, \boldsymbol{\Delta}  \right)$;
        \State $\boldsymbol{y}_{\text{b}} = \texttt{SSM}_{\bar{\mathbf{A}}, \bar{\mathbf{B}}_{\text{b}}, \mathbf{C}_{\text{b}} }\left( \boldsymbol{\Phi}_{\text{b}} \right)$;
        \State $\boldsymbol{Y}_{\text{b}} = \mathbf{W}_{\text{b}, 1}\left(\boldsymbol{y}_{\text{b}}  \odot  \sigma\left( \mathbf{W}_{\text{b}, 2} \: \texttt{LayerNorm}\left( \boldsymbol{\Phi}  \right) \right) \right)$;

        \State \Comment{\third{Output:}$\qquad\qquad\qquad\quad\qquad\qquad\qquad\qquad\:\quad\qquad\qquad\qquad\qquad\qquad\qquad\qquad\qquad\qquad\qquad\qquad\qquad\qquad\qquad\qquad\qquad$}
        
        \State $\boldsymbol{y}_{\text{output}} \leftarrow \mathbf{W}_{\text{out}} \left( \boldsymbol{Y}_{\text{f}} + \texttt{Reverse-row}(\boldsymbol{Y}_{\text{b}}) \right)$;
        \State \Return $\boldsymbol{y}_{\text{output}}$;
    \end{algorithmic}
\end{algorithm}

\section{Details of GMN Architecture: Algorithms}
Algorithm~\ref{alg:GMN} shows the forward pass of the Graph Mamba Network with one layer. For each node, GMN first samples $M$ walks with length $\hat{m} = 1, \dots, m$ and constructs its corresponding tokens, each of which as the induced subgraph of $M$ walks with length $\hat{m}$. We repeat this process $s$ times to have longer sequence and more samples from each hierarchy of the neighborhoods. This part of the algorithm, can be computed before the training process and in CPU. Next, GMNs for each node encode its tokens using an encoder  $\phi(.)$, which can be an MPNN (e.g., gated-GCN~\citep{bresson2017residual}) or RWF encoding (proposed by \citet{nshoff2023walking}). We then pass the encodings to a Bidirectional Mamba block, which we describe in Algorithm~\ref{alg:BiMamba} (This algorithm is simple two Mamba block~\citep{gu2023mamba} such that we use one of the backward or forward ordering of inputs for each of them). At the end of line 15, we have the node encodings obtained from subgraph tokenization. Next, we treat each node as a token and pass the encoding to another bidirectional Mamba, with a specific order. We have used degree ordering in our experiments, but there are some other approaches that we have discussed in the main paper.

\section{Additional Experimental Results}

\subsection{Parameter Sensitivity}

\head{The effect of $M$}
Parameter $M$ is the number of walks that we aggregate to construct a subgraph token. To evaluate its effect on the performance of the GMN, we use two datasets of Roman-empire and PascalVOC-SP, from two different benchmarks, and vary the value of $M$ from 1 to 10. The results are reported in Figure~\ref{fig:param} (Left).  These results show that performance peaks at certain value of $M$ and the exact value varies with the dataset. The main reason is, parameter $M$ determines that how many walks can be a good representative of the neighborhood of a node, and so depends on the density, homophily score, and network topology this value can be different. 

\head{The effect of $m$}
Similar to the above, we use two datasets of Roman-empire and PascalVOC-SP, from two different benchmarks, and vary the value of $m$ from 1 to 60. The results are reported in Figure~\ref{fig:param} (Middle). The performance is non-decreasing with respect to the value of $m$. That is, increasing the value of $m$, i.e., considering far neighbors in the tokenization process, does not damage the performance (might lead to better results). Intuitively, using large values of $m$ is expected to damage the performance due to the over-smoothing and over-squashing; however, the selection mechanism in Bidirectional Mamba can select informative tokens (i.e., neighborhood), filtering information that cause performance damage. 

\head{The effect of $s$}
Using the same setting as the above, we report the results when varying the value of $s$ in Figure~\ref{fig:param} (Right). Result show that increasing the value of $s$ can monotonically improve the performance. As discussed earlier, longer sequences of tokens can provide more context for our model and due to the selection mechanism in Mamba~\citep{gu2023mamba}, GMNs can select informative subgraphs/nodes and filter irrelevant tokens, resulting in better results with longer sequences.

\begin{figure*}
    \centering
    \includegraphics[width=0.33\textwidth]{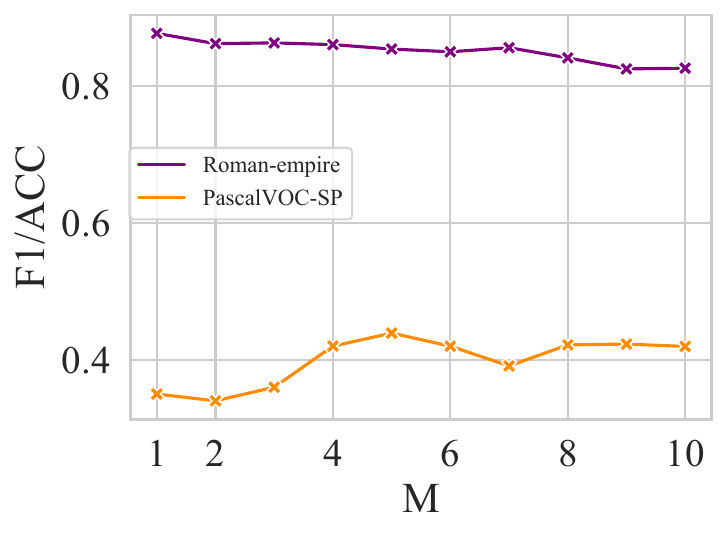}~
    \includegraphics[width=0.33\textwidth]{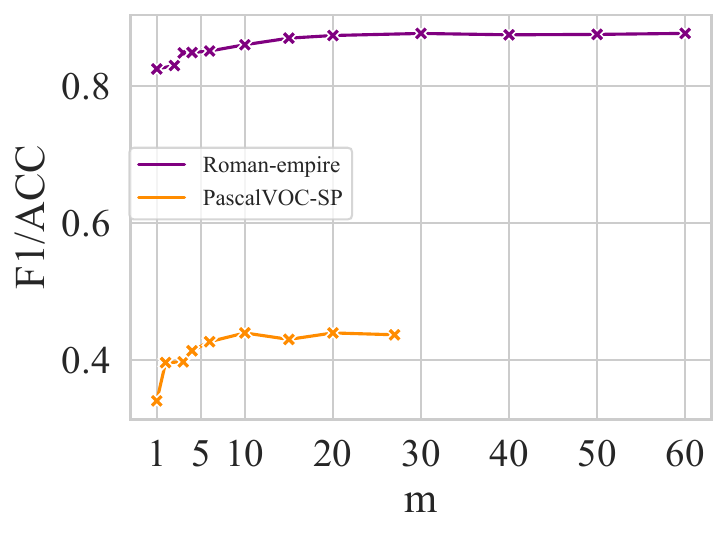}~
    \includegraphics[width=0.33\textwidth]{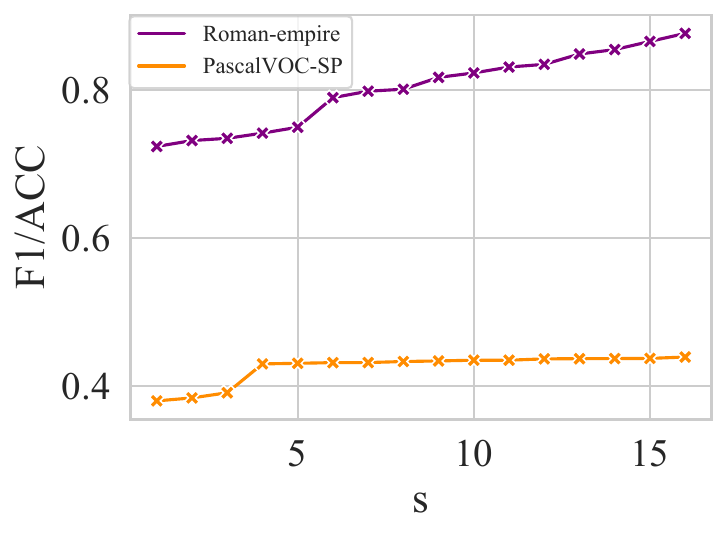}
    \vspace{-3ex}
    \caption{The effect of (Left) $M$, (Middle) $m$, and (Right) $s$ on the performance of GMNs.}
    \label{fig:param}
\end{figure*}

\subsection{Comparison with GRED~\citep{ding2023recurrent} and S4G~\citep{s4g}}
 GRED~\citep{ding2023recurrent} is a recent work on ArXiv that uses an RNN on the set of neighbors with distance $k = 1, \dots, K$ to a node of interest for the node classification task. S4G~\citep{s4g} is a recent unpublished (without preprint but available on Openreview) work that uses the same approach as GRED but replaces the RNN block with a structured SSM. Since the code or models of GRED and S4G are not available, for the comparison of GMNs, S4G, and GRED, we run GMNs on the datasets used in the original papers~\citep{ding2023recurrent, s4g}. The results are reported in Table~\ref{tab:GRED}. GMNs consistently outperforms GRED~\citep{ding2023recurrent} in all datasets and outperforms S4G in 4 out of 5 datasets. The reason is two folds: (1) GMNs use sampled walks instead of all the nodes within $k$-distance neighborhood. As discussed in Theorem~\ref{thm:neighborhood}, this approach with large enough length and samples is more expressive than considering all nodes within the neighborhood. (2) S4G and GRED use simple RNN and SSM to aggregate the information about all the different neighborhoods of a node, while GMNs use Mamba, which have a selection mechanism. This selection mechanism help the model to choose neighborhoods that are more informative and important than others. (3) GRED and S4G are solely based on distance encoding, meaning that they miss the connections between nodes in $k$-distance and $(k+1)$-distance. Figure~\ref{fig:failureExample} shows a failure example of these methods that solely are based on distance of nodes. To obtain the node encoding of node $A$, these two methods group nodes wit respect to their distance to $A$, either $d = 1, 2,$ and $3$. In Figure~\ref{fig:failureExample}, while these two graphs are non-isomorphism, the output of this step for both graphs are the same, meaning that these methods obtain the same node encoding for $A$.

 \begin{figure}
     \centering
     \includegraphics[width=\linewidth]{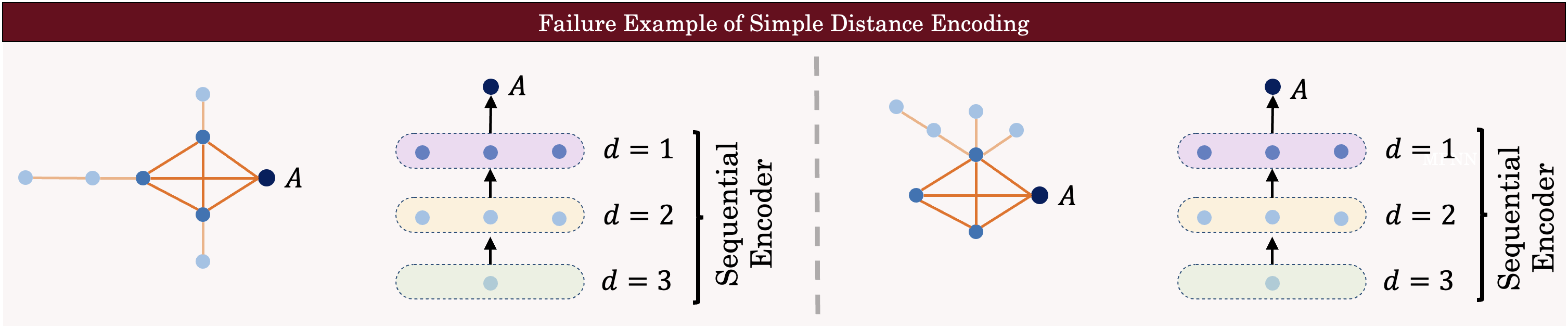}
     \vspace{-2ex}
     \caption{Failure example for methods that are solely based on distance encoding. Solely considering the set of nodes in different distances to the target node misses the connections between them. While the structure of these two graphs are different, the set of nodes with the same distance to node $A$ are the same. Accordingly, GRED~\citep{ding2023recurrent} and S4G~\citep{s4g} achieve the same node encoding for $A$, missing $A$'s neighborhood topology.}
     \label{fig:failureExample}
 \end{figure}

\begin{table}
    \centering
    \caption{Comparison with GRED and S4G Models. Highlighted are the top \first{first} and \second{second} results.}
    \resizebox{0.80\linewidth}{!}{
    \begin{tabular}{l  c  c   c  c  c}
    \toprule
        \multirow{2}{*}{\textbf{Model}} & \textbf{MNIST}  & \textbf{CIFAR10} & \textbf{PATTERN}  & \textbf{Peptides-func} & \textbf{Peptides-struct} \\
         & Accuracy $\uparrow$& Accuracy $\uparrow$ & Accuracy $\uparrow$ & AP $\uparrow$ & MAE $\downarrow$\\
         \midrule
         \midrule
         S4G~\citep{s4g}$^{\dagger}$ & $0.9637_{\pm 0.0017}$ & $0.7065_{\pm 0.0033}$ & \second{$0.8687_{\pm 0.0002}$} & \first{$0.7293_{\pm 0.0004}$} & \second{$0.2485_{\pm 0.0017}$}\\
         GRED~\citep{ding2023recurrent}$^{\ddagger}$ & \second{$0.9822_{\pm 0.0095}$} & \second{$0.7537_{\pm 0.6210}$} & $0.8676_{\pm 0.0200}$ & {$0.7041_{\pm 0.0049}$} & {$0.2503_{\pm 0.0019}$}\\
         \midrule
         GMN (Ours)& \first{$0.9839_{\pm 0.0018}$} & \first{$0.7576_{\pm 0.0042}$} & \first{$0.8714_{\pm 0.0012}$} & \second{$0.7071_{\pm 0.0083}$} & \first{$0.2473_{\pm 0.0025}$} \\
    \toprule
     \multicolumn{6}{l}{$^\dagger$ Results are reported by \citet{s4g}.}\\
    \multicolumn{6}{l}{$^\ddagger$ Results are reported by \citet{ding2023recurrent}.}
    \end{tabular}
    \label{tab:GRED}
    }
\end{table}

\section{Complexity Analysis of GNMs}
\head{\underline{$m \geq 1$}} For each node $v \in V$, we generate $M \times s$ walks with length $\hat{m} = 1, \dots, m$, which requires $\mathcal{O}\left( M \times s \times (m+1) \right)$ time. Given $K$ tokens, the complexity of bidirectional Mamba is $2 \times$ of Mamba~\citep{gu2023mamba}, which is linear with respect to $K$. Accordingly, since we have $\mathcal{O}\left( M \times s \times m \right)$ tokens, the final complexity for a given node $v \in V$ is $\mathcal{O}\left( M \times s \times (m+1) \right)$. Repeating the process for all nodes, the time complexity is $\mathcal{O}\left( M \times s \times (m+1) \times |V|  + |E| \right)$, which is linear in terms of $|V|$ and $|E|$ (graph size). To compare to the quadratic time complexity of GTs, even for small networks, note that in practice, $M \times s \times (m+1) \ll |V|$, and in our experiments usually $M \times s \times (m+1) \leq 200$. Also, note that using MPNN as an optional step cannot affect the time complexity as the MPNN requires $\mathcal{O}\left( |V| + |E| \right)$ time.

\head{\underline{$m = 0$}} In this case, each node is a token and so the GMN requires $\mathcal{O}\left( |V| \right)$ time. Using MPNN in the architecture, the time complexity would be $\mathcal{O}\left( |V| + |E| \right)$, dominating by the MPNN time complexity.

As discussed above, based on the properties of Mamba architecture, longer sequence of tokens (larger value of $s \geq 1$) can improve the performance of the method. Based on the abovementioned time complexity when $m \geq 1$, we can see that there is a trade-off between time complexity and the performance of the model. That is, while larger $s$ result in better performance, it results in slower model.

\section{Discussion on a Concurrent Work}
\citep{GMB}, in work concurrent to and independent of ours, replace Transformer architecture (attention block) with a Mamba block~\citep{gu2023mamba} in GPS framework~\citep{GPS}. Next, we discuss these two works in different aspects:

\head{Architecture Design}
As mentioned above, the GMB model~\citep{GMB} replaces the attention module in GPS framework~\citep{GPS} with Mamba block~\citep{gu2023mamba}. Accordingly, it treats each node as a token, uses PE/SE as initial feature vectors, and ordered nodes based on their degree. Since this approach is based on node tokenization and uses one directional Mamba~\citep{gu2023mamba}, it suffers from the limitations of GTs with node tokenization mentioned in \S~\ref{sec:motivations}. More specifically: (1) Although it has more ability to learn long-range dependencies, this approach lacks inductive bias about the graph structure and requires complex PE/SE. (2) Using a simple one directional Mamba causes the lack of inductive bias about some structures in the graph. Figure~\ref{fig:failureExample2} provides an example of this information loss. In part (A), we show an example of node tokenization and node ordering with respect to their degrees. Based on the information flow, using a one directional Mamba, nodes with high-degree do not have any information about the structure of the graph. For example, in Figure~\ref{fig:failureExample2} (B), even with using complex PE/SE, nodes in the right hand side do not have any information about the global information in the left hand side, due to the one directional information flow of Mamba block. As discussed earlier in the paper, this is one of the new challenges of using SSMs (compare to GTs). The main reason is, attentions in Transformers consider all nodes connected and so the information could pass between each pair of nodes. In sequential encoders (even with selection mechanism), however, each token has the information about its previous tokens. Our neighborhood sampling and its reverse ordering can address this issue due to its implicit order of neighborhood hierarchy. 

Comparing to GMNs, GMB can be seen as a special case of GMNs when we use $m = 0$ and replace bidirectional Mamba block with a one directional Mamba block. Our approach using parameter $m$ provides the flexibility of using either node or subgraph tokenization, whenever either inductive bias or long-range dependencies is more important to the task and the dataset. Furthermore, having these two special traits of GMNs compared to GMB results in provable expressive power of GMNs, which we discuss in the following.

 \begin{figure}
     \centering
     \includegraphics[width=0.9\linewidth]{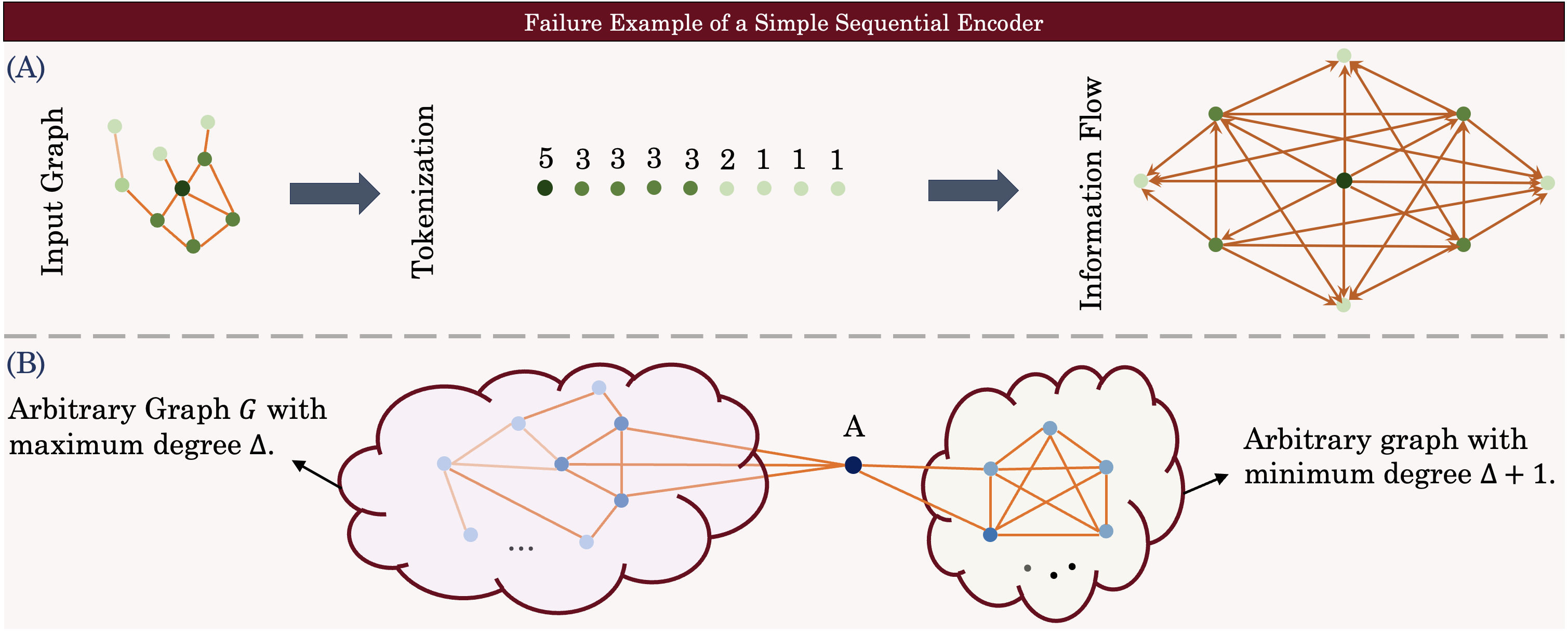}
     \vspace{-2ex}
     \caption{(A) An example of node tokenization and its information flow. Even using PE/SE, nodes at the beginning of the sequence do not have any information about the structure of the graph! (B) Potential failure example for using a simple one directional sequential encoder when each token is a node. Nodes in the right hand side do not have any information about the structure of the graph in the left hand side (Or vise versa depends on the direction of the order). }
     \label{fig:failureExample2}
 \end{figure}

\head{Expressive Power}
As discussed earlier, due to the lack of inductive bias, GMB method requires complex PE/SE to learn about the structure of the graph. While GMNs \emph{without} PE/SE and MPNNs has unbounded expressive power with respect to isomorphic test (Theorem~\ref{thm:unbounded}), GMB cannot distinguish graphs with the same sequence of degrees and its expressive power is bounded by the expressive power of its MPNN. That is, let $G_1$ and $G_2$ be two graphs with the same sequence of node degrees, Mamba block in GMB, in the worst case of having the same node feature vectors, cannot distinguish $G_1$ and $G_2$ since its input is the same for both of these graphs. Accordingly, the expressive power of GMB is bounded by the expressive power of its MPNN in the GPS framework~\citep{GPS}.

\section{Theoretical Analysis of GMNs}\label{app:theory}

\begin{theorem}
    With large enough $M, m,$ and $s > 0$, GMNs' neighborhood sampling is strictly more expressive than $k$-hop neighborhood sampling.
\end{theorem}
\begin{proof}
    We first show that in this condition, the random walk neighborhood sampling is as expressive as $k$-hop neighborhood sampling. To this end, given an arbitrary small $\epsilon > 0$, we show that the probability that $k$-hop neighborhood sampling is more expressive than random walk neighborhood sampling is less than $\epsilon$. Let $m = k$, $s = 1$, and $p_{u, v}$ be the probability that we sample node $v$ in a walk with length $m = k$ starting from node $u$. This prbobality is zero if the shortest path of $u$ and $v$ is more than $k$. To construct the subgraph token corresponds to $\hat{m} = k$, we use $M$ samples and so the probability of not seeing node  $v$ in these samples is $q_{u,v, M} = (1 - p_{u, v})^{M} \leq 1$. Now let $M \rightarrow \infty$ and $v \in \mathcal{N}_{k}(u)$ (i.e., $ p_{u, v} \neq 0$), we have $\lim_{M \rightarrow \infty} q_{u,v, M} = 0$. Accordingly, with large enough $M$, we have $q_{u,v, M} \leq \epsilon$. This means that with a large enough $M$ when $m = k$ and $s = 1$, we sample all the nodes within the $k$-hop neighborhood, meaning that random walk neighborhood sampling at least provide as much information as $k$-hop neighborhood sampling with arbitrary large probability. 
    \\
    Next, we provide an example that $k$-hop neighborhood sampling is not able to distinguish two non-isomorphism graphs, while random walk sampling can. Let $S = \{v_1, v_2, \dots, v_{\ell} \}$ be a set of nodes such that all nodes have shortest path less than $k$ to $u$. Using hyperparamters $m = k$ and arbitrary $M$, let the probability that we get $G[S]$ as the subgraph token be $1 > q_S > 0$. Using $s$ samples, the probability that we do not have $G[S]$ as one of the subgraph tokens is $(1 - q_S)^s$. Now using large $s \rightarrow \infty$, we have $\lim_{s \rightarrow \infty} (1 - q_S)^s = 0$ and so for any arbitrary $\epsilon > 0$ there is a large $s > 0$ such that we see all non-empty subgraphs of the $k$-hop neighborhood with probability more than $1 - \epsilon$, which is more powerful than the neighborhood itself. 
    \\
    Note that the above proof does not necessarily guarantee an efficient sampling, but it guarantees the expressive power.   
\end{proof}

\begin{theorem}[Universality]
    Let $1 \leq p < \infty$, and $\epsilon > 0$. For any continues function $f: [0, 1]^{d \times n} \rightarrow \R^{d \times n}$ that is permutation equivariant, there exists a GMN with positional encoding, $g_p$, such that $\ell^p(f, g) < \epsilon$. 
\end{theorem}

\begin{proof}
    Recently,  \citet{wang2023statespace} show that SSMs with layer-wise nonlinearity are universal approximators of any sequence-to-sequence function. We let $m = 0$, meaning we use node tokenization. Using the universality of SSMs for sequence-to-sequence function, the rest of the proof is the same as \citet{kreuzer2021rethinking}, where they use the padded adjacency matrix of $G$ as a positional encoding to prove the same theorem for Transformers. In fact, the universality for sequence-to-sequence functions is enough to show the universality on graphs with a strong positional encoding. 
\end{proof}

\begin{theorem}[Expressive Power w/ PE]\label{thm:universality}
    Given the full set of eigenfunctions and enough parameters, GMNs can distinguish any pair of non-isomorphic graphs and are more powerful than any WL test.
\end{theorem}

\begin{proof}
    Due to the universality of GMNs in Theorem~\ref{thm:universality}, one can use a GMN with the padded adjacency matrix of $G$ as positional encoding and learn a function that is invariant under node index permutations and maps non-isomorphic graphs to different values. 
\end{proof}

\begin{theorem}[Expressive Power w/o PE and MPNN]
    With enough parameter, 
    for every $k \geq 1$ there are graphs that are distinguishable by GMNs, but not by $k$-WL test, showing that their expressivity power is not bounded by any WL test.
\end{theorem}

\begin{proof}
    The proof of this theorem directly comes from the recent work of \citet{nshoff2023walking}, where they prove a similar theorem for CRaWl~\citep{nshoff2023walking}. That is, using RWF as $\phi(.)$ in GMNs (without MPNN), makes the GMNs as powerful as CRaWl. The reason is CRaWl uses 1-d CNN on top of RWF, while GMNs use Bidirectional Mamba block on top of RWF: using a broadcast SMM, this block becomes similar to 1-d CNN.

\end{proof}

\end{document}